\documentclass[twoside,11pt]{article}

\usepackage{blindtext}

\usepackage{jmlr2e}

\usepackage{lastpage}
\usepackage{amsmath}
\usepackage{nicefrac}
\usepackage{tocloft}
\usepackage{etoolbox}
\usepackage{enumitem}
\usepackage{indentfirst}
\usepackage{caption}
\usepackage{lmodern}
\usepackage{nicefrac}
\usepackage{fancyhdr}
\usepackage[usenames,dvipsnames]{xcolor}
\usepackage{subfig}

\newcommand \dd[1]  { \,\mathrm d{#1}}
\newcommand{\PX}{P_{\scriptscriptstyle X}}
\newcommand{\QX}{Q_{\scriptscriptstyle X}}

\newcommand{\cl}[1]{\textcolor{black}{#1}} 

\newcommand{\QY}{Q_{\scriptscriptstyle Y}}

\newcommand{\BP}{\mathbb{P}}
\newcommand{\BX}{\mathbb{X}}

\newcommand{\BY}{\mathbb{Y}}

\newcommand{\BI}{\mathbb{I}}
\newcommand{\CX}{\mathcal{X}}

\newcommand{\CF}{\mathcal{F}}
\newcommand{\CA}{\mathcal{A}}
\newcommand{\CB}{\mathcal{B}}

\def \CG{\mathcal{G}}
\def \CF{\mathcal{F}}
\def \CX{\mathcal{X}}
\def \CY{\mathcal{Y}}

\usepackage{mathtools}

\DeclareMathOperator {\Var}{var}

\newcommand{\tr}{^\mathrm{T}}  
\newcommand{\RR}{\mathbb{R}}

\newcommand{\NN}{\mathbb{N}}

\newcommand{\EE}{\mathbb{E}}
\newcommand{\PP}{\mathbb{P}}
\newcommand{\norm}[1]{\left\lVert#1\right\rVert}

\DeclareFontFamily{U}{mathx}{\hyphenchar\font45}
\DeclareFontShape{U}{mathx}{m}{n}{<-> mathx10}{}
\DeclareSymbolFont{mathx}{U}{mathx}{m}{n}
\DeclareMathAccent{\widebar}{0}{mathx}{"73}

\newtheorem{innercustomgeneric}{\customgenericname}
\providecommand{\customgenericname}{}
\newcommand{\newcustomtheorem}[2]{\newenvironment{#1}[1]
	{\renewcommand\customgenericname{#2}\renewcommand\theinnercustomgeneric{##1}\innercustomgeneric
	}
	{\endinnercustomgeneric}
}

\newcustomtheorem{customthm}{Theorem}
\newcustomtheorem{customlemma}{Lemma}

\allowdisplaybreaks

\ShortHeadings{Recursive Estimation of Conditional Kernel Mean Embeddings}{Tam{\'a}s and Cs{\'a}ji}
\firstpageno{1}

\begin{document}
\thispagestyle{empty}

\title{Recursive Estimation of\\ Conditional Kernel Mean Embeddings}
	
	\author{\name Ambrus Tam{\'a}s${}^{1,2}$ \email ambrus.tamas@sztaki.hu \\[1mm]
		\name Bal{\'a}zs Csan{\'a}d Cs{\'a}ji${}^{1,2}$ \email balazs.csaji@sztaki.hu\\[2mm]
        \addr ${}^{1}$Institute for Computer Science and Control (SZTAKI)\\
Hungarian Research Network (HUN-REN),\\
		Kende utca 13-17, Budapest, H-1111, Hungary\\[2mm]  
        \addr ${}^{2}$Department of Probability Theory and Statistics\\
        Institute of Mathematics, E{\"o}tv{\"o}s Lor{\'a}nd University (ELTE)\\
        P\'azm\'any P\'eter S\'et\'any 1\,/\,C, Budapest, H-1117, Hungary
	}
	
	\editor{}
	
	\newcounter{asscounter}
	\setcounter{asscounter}{1}
	
	\maketitle
	
	\begin{abstract}Kernel mean embeddings, a widely used technique in machine learning, map probability distributions to elements of a reproducing kernel Hilbert space (RKHS). For supervised learning problems, where input-output pairs are observed, the conditional distribution of outputs given the inputs is a key object. The input dependent conditional distribution of an output can be encoded with an RKHS valued function, the conditional kernel mean map. In this paper we present a new recursive algorithm to estimate the conditional kernel mean map in a Hilbert space valued $L_2$ space, that is in a Bochner space. We prove the weak and strong $L_2$ consistency of our recursive estimator under mild conditions. The idea is to generalize Stone's theorem for Hilbert space valued regression in a locally compact Polish space. We present new insights about conditional kernel mean embeddings and give strong asymptotic bounds regarding the convergence of the proposed recursive method. Finally, the results are demonstrated on three application domains: for inputs coming from Euclidean spaces, Riemannian manifolds and locally compact subsets of function spaces.
	\end{abstract}
    \vspace{1mm}
    
    
    \section{Introduction}
    
    In {\em statistical learning} we need to work with probability distributions \cl{defined on a wide variety of objects, which may not even come from a linear space}.\
{\em Kernel mean embeddings} (KMEs) offer a neat representation for dealing with distributions by mapping them into a {\em reproducing kernel Hilbert space} (RKHS). These embeddings were defined in \citep{berlinet2004reproducing} and \citep{smola2007hilbert}. In machine learning often the conditional distribution of an output w.r.t.\ some input plays a key role. In these problems the concept of {\em conditional kernel mean embeddings} (CKME) is more adequate. The notion of CKME was first recognized in \citep{song2009a}, then a refined definition was given in \citep{song2013kernel}. Its rigorous measure theoretic background was presented in \citep{klebanov2020rigorous} and in \citep{park2020measure} as an input dependent random element in an RKHS. These concepts strongly rely on the notions of Bochner integral and conditional expectation in Banach spaces of which theory is presented in \citep{pisier2016martingales,hytonen2016analysis} and Appendix A. Under some stringent assumptions \cite{song2009a} estimate the conditional kernel mean map with an empirical variant of the regularized inverse cross-covariance operator. It is shown in \citep{grunewalder2012} that this estimate can also be deduced from a regularized risk minimization scheme in a {\em vector-valued} RKHS based on \citep{michelli2005}. In \citep{park2020measure} a consistency theorem is proved when the conditional kernel mean map is included in a vector-valued RKHS.
    
    \paragraph{Related works}
    In this paper we follow the general {\em distribution-free} framework presented in \citep{gyorfi2002distribution}. It is recognized that estimating the conditional kernel mean map is a particular {\em vector-valued regression} problem. Vector-valued regression was already studied in \citep{ramsay2008functional,bosq1985nonparametric,dabo2009kernel,forzani2012consistent,brogat2022vector}. Nevertheless, it is still challenging to provide sufficient conditions for strong $L_2$ (mean square) consistency that can be applied in practice. We also build on the theory of {\em stochastic approximation} \cl{\citep{kushner2003stochastic,ljung2012stochastic}}. Our core algorithm uses a stochastic update in each iteration. The presented method was motivated by the celebrated {\em stochastic gradient descent} algorithm, see for example \citep[Proposition 4.2]{bertsekas1996neuro}. The update term of our recursive algorithm is a modified gradient of an {\em empirical surrogate loss} function.

    \paragraph{Contributions}
    Our main contribution is a {\em weakly and strongly consistent recursive algorithm} for estimating the conditional kernel mean map under general structural assumptions. 
    We present and prove a {\em generalized version of Stone's theorem} \citep{stone1977consistent}, motivated by the recursive form in \citep[Theorem 25.1]{gyorfi2002distribution}. Our generalization is twofold. We deal with general {\em locally compact Polish input spaces} and allow the output space to be {\em Hilbertian}. Our main assumption is that the {\em measure of metric balls} goes to zero slowly w.r.t.\ the radius. In \citep{forzani2012consistent} a similar consistency result is proved, however the authors of that paper do not deal with functional outputs and also require the so-called Besicovitch condition to hold on the regression function. In \citep{dabo2009kernel} sufficient conditions for $L_2$ consistency are given for vector-valued regression, however, in this paper a differentiation condition is required for the regression function and also the rate of the small ball probabilities is bounded from below when the radius goes to zero. 

    First, \cl{in Section 2} we present a recursive scheme for (unconditional) kernel mean embeddings. A stochastic approximation based theorem \citep{ljung2012stochastic} is used to prove its strong consistency. Then, \cl{in Section 3} we introduce a recursive estimator for the conditional kernel mean map. Our algorithm is formulated in a flexible offline manner in such a way that the online algorithm is also covered. We prove the weak $L_2$ consistency for the general {\em local averaging} scheme in locally compact Polish spaces under very mild assumptions. We apply this consistency result to prove the strong consistency of local averaging kernel estimates with a (measurable) nonnegative, bounded and symmetric {\em smoother} sequence. This 
scheme requires a smoother sequence with appropriate shrinking properties w.r.t.\ the unknown marginal measure. We satisfy these requirements by a general mother smoother function induced sequence parameterized by two shrinkage rates. In the final theorem only a {\em structural condition} on the probability of small balls remains. Finally, \cl{in Section 4} we apply our results for three arch-typical cases. We deduce universal consistency for {\em Euclidean 
spaces}, {\em Riemannian manifolds} and locally compact subsets of 
{\em Hilbert spaces}.

    \paragraph{Applications}
    Here, we highlight some potential application domains of the ideas. Based on the theory of KMEs and CKMEs, a rich variety of applications were introduced. Kernel mean embeddings can be efficiently used, for example, in graphical models \citep{song2013kernel}, Bayesian inference \citep{fukumizu2013kernel}, dynamical systems \citep{song2009a}, reinforcement learning \citep{grunewalder2012modelling}, causal inference \citep{scholkopf2015computing} and hypothesis testing for independence \citep{gretton2007kernel,gretton2008nonparametric,gretton2012kernel}, conditional independence \citep{fukumizu2007kernel,Zhang2011}, and binary classification \citep{csaji2019} and \citep{tamas2021}.

    \cl{Recursive algorithms are crucial to handle data streams or fixed, but very large datasets, cf.\ divide-and-conquer. Iterative methods play a key role in, e.g., reinforcement learning, system identification, signal processing and adaptive control. Therefore, developing recursive estimators for conditional kernel mean embeddings is beneficial for many domains.}

	\section{Kernel Mean Embeddings of Distributions}
    Kernel methods offer an elegant approach to deal with abstract sample spaces. One only needs to specify a {\em similarity measure} on the space via the {\em kernel} function to provide a geometrical structure for the learning problem. In fact, every symmetric, positive definite function, a.k.a.\ kernel, defines a reproducing kernel Hilbert space. The user-chosen kernel function also determines a natural projection from the sample space into the RKHS. The projected values of sample points are called feature vectors, or {\em features}. This projection map can be easily extended to (probability) measures defining kernel mean embeddings. In this section we present the notion of KMEs in a statistical learning framework.
    
	\subsection{Statistical Framework}
	Let $(\Omega, \CA, \BP)$ be a probability space, where $\Omega$ is the sample space, $\CA$ is the $\sigma$-algebra of events, and $\BP$ is the probability measure. We consider a random pair $(X,Y)$ from a measurable product space $\BX \times \BY$ with some product $\sigma$-algebra $\CX \otimes \CY$. Let us consider a metric on $\BX$ and let us denote its induced Borel $\sigma$-algebra as $\CX$. We denote the joint distribution of $(X,Y)$ by $Q$ (or $Q_{{\scriptscriptstyle{X,Y}}}$) and the corresponding marginal distributions by $Q_{\scriptscriptstyle{X}}$ and $Q_{\scriptscriptstyle{Y}}$. In practice, distribution $Q$ is usually {\em unknown}, however we assume that:
	\begin{itemize}
		\item[A\theasscounter] {\em An independent identically distributed {\em(}i.i.d.{\em)} sample $\{(X_i,Y_i)\}_{i=1}^n$ is given.} \addtocounter{asscounter}{1}
	\end{itemize}

    In the next section we present the definition of kernel mean embeddings of probability distributions \cl{into an RKHS associated with kernel $\ell:\BY \times \BY \to \RR$.} With a somewhat imprecise terminology, we also talk about the kernel mean embedding of a random variable into an RKHS, by which we mean taking the KME of its distribution. Explanatory variable $X$ will only be used for the {\em conditional} variant of KME to generate the conditioning $\sigma$-algebra, consequently, for the {\em unconditional} KME, we will only use one variable, $Y$.

    In the subsequent section we focus on conditional kernel mean embeddings. The conditional kernel mean embedding of $Y$ with respect to $X$ into a given RKHS is the conditional expectation of the random feature defined by the kernel and variable $Y$ with respect to the $\sigma$-algebra generated by $X$, i.e., a conditional kernel mean embedding is a random feature vector in an RKHS measurable to some explanatory variable. The main object for CKME is the {\em conditional kernel mean map} which maps the explanatory points into RKHS $\CG$. This function maps (almost) every input point to the kernel mean embedding of the conditional distribution given that input. For a linear kernel one can recover the regression function as a particular conditional kernel mean map. When $X$ and $Y$ are independent, the conditional kernel map is a constant function mapping every input into the unconditional KME of $Y$.

	\subsection{Reproducing Kernel Hilbert Spaces}

	Let $\ell$ be a \cl{(real-valued, symmetric, positive definite)} kernel on $\BY \times \BY$, where $\BY$ is a nonempty set, \cl{i.e., for all $n \in \NN\doteq \{1, 2, \dots\}$,} $y_1, \dots, y_n \in \BY$ and $a_1,\dots,a_n \in \RR$, we have
	\begin{equation}
		\begin{aligned}
			\sum_{i=1}^n \sum_{j=1}^n a_i a_j \, \ell(y_i,y_j)\,\geq\, 0,
		\end{aligned}
	\end{equation}
	or equivalently kernel matrix $L \in \RR^{n \times n}$, where $L_{i,j}\doteq  \ell(y_i,y_j)$ for $i$, $j \in [n] \doteq \{ 1, \dots, n\}$, is required to be positive semidefinite. A Hilbert space of functions, $\CG= \{g: \BY \to \RR\}$, is called a {\em reproducing kernel Hilbert space} if every point evaluating (linear) functional, $\delta_y: \CG \to \RR$ defined by $\delta_y(g)=g(y)$ for $y \in \BY$, is bounded (or equivalently continuous).  In \citep{aronszajn1950theory} it is proved that for every \cl{(symmetric and positive definite)} kernel $\ell$ there uniquely (up to isometry) exists an RKHS $\CG$, such that $\ell(\cdot, y) \in \CG$ and
	\begin{equation}\label{reproducing-property}
	\begin{aligned}
		\langle\hspace{0.7mm} g,\hspace{0.3mm} \ell(\cdot, y) \hspace{0.2mm}\rangle_\CG \,=\, g(y)
	\end{aligned}
	\end{equation} 
	holds for all $g \in \CG$ and $y\in \BY$. In particular, $\langle \ell(\cdot, y_1), \ell(\cdot, y_2) \rangle_\CG = \ell(y_1, y_2)$ is true. In the reverse direction, by the Riesz representation theorem \cl{\cite[Theorem 4]{lax2002functional}}, for every RKHS $\CG$ there uniquely exists a \cl{(symmetric and positive definite) kernel} $\ell: \BY \times \BY \to \RR$ such that (\ref{reproducing-property}) holds for all $g \in \CG$ and $y\in \BY$. Function $\ell$ is called the {\em reproducing kernel} of space $\CG$. For further details, see \citep{scholkopf2002learning,berlinet2004reproducing,shawe2004kernel,steinwart2008support}.
	
	\subsection{Kernel Mean Embeddings of Marginal Distributions}
	
	In this paper we will assume that space $\CG$ is {\em separable}. One can observe that if $\BY$ is a separable topological space and kernel $\ell$ is continuous, then the separability of $\CG$ follows immediately, see \cite[Lemma 4.33]{steinwart2008support}. Specifically, we assume
\begin{itemize}
    	\item[B1] {\em Kernel $\ell$ is measurable, $\EE\big[\sqrt{\ell(Y,Y)}\,\big] < \infty$, and RKHS $\CG$ is separable.}
	\end{itemize}
    In this case, by the theorem of Pettis \cite[Theorem 1.1.20]{hytonen2016analysis}, $\ell(\cdot, Y):\Omega \to \CG$ is also measurable.
	\cl{Then, the kernel mean embedding of $\mu_{\scriptscriptstyle Y}$ exists if and only if $B1$ holds \cite[Proposition 1.2.2]{hytonen2016analysis}. Thus, one can embed
$\QY$ into RKHS $\CG$ by}
	\begin{equation}
		\mu_{\scriptscriptstyle{Y}} \doteq\, \EE \big[\, \ell(\cdot, Y)\,\big],
	\end{equation}
	where the expectation is a {\em Bochner integral}. It is known \citep{smola2007hilbert} that if $B1$ holds, then by the Riesz representation theorem $\mu_{\scriptscriptstyle{Y}} \in \CG$ and for all $g \in \CG$, we have
    \begin{equation}
    \langle \mu_Y , g \rangle_\CG \,=\, \EE \big[\,g(Y) \,\big],
    \end{equation}
    i.e., $\mu_{\scriptscriptstyle{Y}}$ is the representer of the {\em expectation functional} 
    on $\CG$ w.r.t.\ $\QY$.
	
	\cl{Given an i.i.d.}\ sample $\{Y_i\}_{i=1}^n$, having distribution $Q_{\scriptscriptstyle{Y}}$, we can estimate the kernel mean embedding of $\QY$ with the empirical average (a.k.a.\ sample mean), that is
	\begin{equation}
		\widehat{\mu}_n \;\doteq\; \frac{1}{n}\, \sum_{i=1}^n \ell(\cdot, Y_i).
	\end{equation}
	By the strong law of large numbers, $\widehat{\mu}_n \to \mu_{\scriptscriptstyle{Y}}$ almost surely \citep{taylor1978stochastic}. We refer to \cite[Proposition A.1.]{tolstikhin2017minimax} to quantify the convergence rate of this estimate:
	\begin{theorem}\label{theorem:SLLN_emb}
		Let $\{Y_i\}_{i=1}^n$ be an i.i.d.\ sample from distribution $Q_{\scriptscriptstyle{Y}}$ on a separable topological space $\BY$. Assume that $\ell(\cdot,y)$ is continuous and there exists $C_\ell > 0$ such that $\ell(y,y) \leq C_\ell$ for all $y \in \BY$. Then, for all $\delta \in (0,1)$ with probability at least $1-\delta$, we have
		\begin{equation}
			\norm{\mu_{\scriptscriptstyle{Y}} -\widehat{\mu}_{n}}_\CG \,\leq\, \sqrt{\frac{C_\ell}{n}}+ \sqrt{ \frac{2\, C_\ell \log\big(\frac{1}{\delta}\big)}{n}}.
		\end{equation} 
	\end{theorem}
	It is also proved that this rate is minimax for a broad class of distributions. Note that only the measurability of $\ell(\cdot,y)$ is used in the proof of \cite{tolstikhin2017minimax}.
	
	\subsection{Recursive Estimation of Kernel Mean Embeddings}
	
	As in the scalar case, these empirical estimates can be computed recursively by
	\begin{equation}
		\begin{aligned}
			&\widehat{\mu}_1 \;\doteq\; \ell(\cdot, Y_1), \quad\text{and}\\[1mm]
			&\widehat{\mu}_n \;\doteq\; \frac{n-1}{n}\,\widehat{\mu}_{n-1} + \frac{1}{n}\,\ell(\cdot,Y_n)\, =\, \widehat{\mu}_{n-1} + \frac{1}{n}\,\big(\ell(\cdot,Y_n) - \widehat{\mu}_{n-1}\big),
		\end{aligned}
	\end{equation}
	for $2\leq n \in \NN$.
	More generally, one can use some possibly random stepsize (or learning rate) sequence $(a_n)_{n\in\NN}$, taking values in $(0,1)$, and apply the refined recursion defined by
	\begin{equation}\label{eq:recursion1}
		\begin{aligned}
			&\widehat{\mu}_1 \;\doteq\; \ell(\cdot, Y_1), \quad\text{and}\\[1.5mm]
			&\widehat{\mu}_n \;\doteq\; \widehat{\mu}_{n-1} - a_n\, \big(\widehat{\mu}_{n-1} - \ell(\cdot,Y_n) \big),
		\end{aligned}
	\end{equation}
    for $n \geq 2$.
 	We prove the strong consistency of $(\widehat{\mu}_n)_{n \in \NN}$ by applying Theorem 1.9 of \citet{ljung2012stochastic}, which is restated (for convenience) as Theorem \ref{thm:RM} in Appendix D.
	\begin{theorem}\label{thm:10}
Let $Y_1, Y_2, \dots$ be an i.i.d.\ sequence, $C_\ell > 0$ be a constant such that $\ell(y,y) \leq C_\ell$, for all $y \in \BY$, and assume B1.
Let $(\widehat{\mu}_n)_{n \in \NN}$ be the estimate sequence defined in \eqref{eq:recursion1}.
        If  $a_n \geq 0$, $\sum_n a_n^2 < \infty$ and $\sum_n a_n = \infty$ {\em(}a.s.{\em)}, then $\widehat{\mu}_n \to \mu_{\scriptscriptstyle Y}$, as $n \to \infty$, almost surely.
	\end{theorem}
	\begin{proof}
       It is sufficient to check the conditions of Theorem \ref{thm:RM} for $\Lambda(g) \doteq g - \mu_{\scriptscriptstyle{Y}} $, $\mu \doteq \mu_{\scriptscriptstyle{Y}}$, \cl{$\mu_n \doteq \widehat{\mu}_n$} and $W_n \doteq  \ell(\cdot, Y_n) - \mu_{\scriptscriptstyle{Y}}$. The conditions on sequence $(a_n)_{n\in \NN}$ are assumed, hence we only need to verify $(i) , (ii)$ and $(iii)$ of Theorem \ref{thm:RM}. Condition $(i)$ holds with $c \doteq \max\big({\norm{\mu_{\scriptscriptstyle{Y}}}_\CG, 1}\big)$. Condition $(ii)$ is again straigthforward, because $\langle \Lambda(g) , g - \mu_{\scriptscriptstyle{Y}}\rangle_\CG = \norm{g - \mu_{\scriptscriptstyle{Y}}}^2_\CG$. For condition $(iii)$ let $H_n \doteq 0$ and $V_n \doteq W_n$. Using the independence of \cl{$\{Y_i\}_{i=1}^n$} yields
		\begin{equation}
			\begin{aligned}
				&\EE\big[V_n \,|\, \widehat{\mu}_1, V_1, \dots, \widehat{\mu}_{n-1}, V_{n-1}\big] = \EE\big[\ell(\cdot, Y_n) - \mu_{\scriptscriptstyle{Y}} \,|\, Y_1, \dots, Y_{n-1} \big] \\[2mm]
				&= \EE\big[\ell(\cdot, Y_n)\big] - \mu_{\scriptscriptstyle{Y}} = 0,
			\end{aligned}
		\end{equation}
		hence $(V_n)_{n\in \NN}$ is a martingale difference sequence. Finally
		\begin{equation}
            \begin{aligned}
		&\sqrt{\EE\Big[\cl{\norm{V_n}^2_\CG}\Big]} \leq  \sqrt{\EE\Big[\norm{\ell(\cdot, Y_n)}^2_\CG\Big]} +            
            \sqrt{\EE\Big[\norm{\mu_{\scriptscriptstyle{Y}}}^2_\CG\Big]} \\
            &\cl{=}\; \sqrt{\EE\Big[\ell(Y_n, Y_n)\Big]} + \norm{\mu_{\scriptscriptstyle{Y}}}_\CG \leq \sqrt{C_\ell} + \norm{\mu_{\scriptscriptstyle{Y}}}_\CG,
		\end{aligned}
            \end{equation}
		thus $\sum_n a_n^2 \EE \big[\norm{V_n}^2_\CG\big] < \infty$. Consequently, Theorem \ref{thm:10} follows from Theorem \ref{thm:RM}.
	\end{proof}

	\section{Conditional Kernel Mean Embeddings}
	
    If the kernel mean embedding of $\QY$ exists, then for a given $\sigma$-algebra one can take the conditional expectation of $\ell(\cdot, Y)$. The conditional kernel mean embedding of $Y$ given $X$ in RKHS $\CG$ is an $X$ measurable random element of $\CG$ defined by
	\begin{equation}
		\begin{aligned}\label{eq:cond_mean}
			&\mu_{{\scriptscriptstyle{Y|X}}} \,\doteq\, \EE \big[\, \ell(\cdot, Y) \,|\, X\, \big].
		\end{aligned}
	\end{equation}
	For more details on conditional expectations in Banach spaces, see \citep{pisier2016martingales} and Appendix A.
	It is known that the conditional expectation is well-defined if $\EE[\,\ell(\cdot,Y)\,]$ exists, thus we assume that $\EE[\,\sqrt{\ell(Y,Y)}\,] < \infty$ as above.
	By \cite[Lemma 1.13]{kallenberg2002foundations} there exists a measurable map $\mu_*: \BX \to \CG$ such that $\mu_{{\scriptscriptstyle{Y|X}}} = \mu_*(X)$ (a.s.). One can use the pointwise form of $\mu_*$ defined $\QX$-a.s.\ by $\mu_*(x) = \EE[\,\ell(\cdot, Y) \,|\, X= x\,]$. This function plays a key role, and we call it the {\em conditional kernel mean map}. By equation \eqref{cont-innerprod} we have that
	\begin{equation}\label{eq5}
		\langle \mu_{{\scriptscriptstyle{Y|X}}}, g \rangle_\CG = \langle \mu_*(X), g \rangle_\CG =\langle\, \EE \big[ \ell(\cdot , Y) \,|\, X \big] , g \rangle_\CG = \EE \Big[ \langle \ell(\cdot, Y), g \rangle_\CG\,|\, X \Big] = \EE \big[\, g(Y) \,|\, X \,\big]
	\end{equation}
	holds for all $g \in \CG$. In addition, by the law of total expectation we have $\EE [\,\mu_{{\scriptscriptstyle{Y|X}}}\,] = \mu_{\scriptscriptstyle{Y}}$.

        \paragraph{Recovering the regression function} 
        Let us consider the special case when $\BX \doteq \RR^d$, $\BY \doteq \RR$  and $\ell(y_1,y_2) = y_1 \cdot y_2$. Then, the RKHS generated by $\ell$ can be simply  identified with $\RR$ by mapping $\ell (y, a) \doteq y \cdot a$ to $a$ and the conditional kernel mean map takes the form of $\cl{(\mu_*(X))(y)} = \EE[\,Y \cdot y \, |\, X\,] \cl{\;= y \, \EE[\,Y \, |\, X\,]}$. Therefore, by linearity, $\mu_*$ is equivalent to the regression function, i.e., $\mu_*(X)$ is the linear function determined by $\EE[\hspace{0.4mm}Y\hspace{0.3mm}|\hspace{0.3mm} X\hspace{0.4mm}]$.
        \cl{
        \paragraph{Recovering a family of real-valued regression}
        For a function $g \in \CG$, the regression of $g(Y)$ given $X$ is a real-valued regression problem for which the results of \cite{gyorfi2002distribution} can be applied. One of the main advantages of dealing with CMKEs is to produce solutions for every $g \in \CG$ simultaneously. Moreover, the information about the conditional distribution of the outputs, given an input, can be used for various types of inference.
        }
        
	\subsection{Universal Consistency}
	
	The main challenge addressed by the paper is to construct an estimate sequence $(\mu_n)_{n\in \NN}$ for $\mu_*$ with strong stochastic guarantees. Since the probability distribution of the given sample is unknown, in general our estimate will not be equal to $\mu_*$ in finite steps. In order to measure the loss of an estimate sequence, one may choose from several error criteria. In this paper we use the vector valued $L_2$ risk with respect to $Q_{\scriptscriptstyle{X}}$. The main reason leading to this choice of $L_2$ criterion is that conditional expectation can be defined as a projection in an $L_2$ space. We strengthen our assumption to ensure that the examined functions are included in the appropriate Bochner spaces. Appendix A
contains an overview on Bochner spaces and Bochner integrals, where our notations are precisely defined.

	\begin{itemize}
		\item[A2] {\em Kernel $\ell$ is measurable, $\EE\big[\,\ell(Y,Y)\,\big] < \infty$, and RKHS $\CG$ is separable}.
	\end{itemize}

	\begin{lemma}
		If $A2$ holds, then $\mu_* \in  L_2(\BX, Q_{\scriptscriptstyle{X}};\CG)$.
	\end{lemma}
	\begin{proof}
	We can see that from $A2$ it follows immediately that $\ell(\cdot, Y) \in L_2(\Omega,\BP;\CG)$ because
	\begin{equation}
	\begin{aligned}
		\int_\Omega \norm{\ell(\cdot, Y)}_\CG^2 \cl{\dd \PP }= \EE \big[\,\ell(Y,Y)\,\big] < \infty.
	\end{aligned}
	\end{equation}
	Moreover by \eqref{eq:nonexpansive}, \cite[Prop. 2.6.31.]{hytonen2016analysis},  we have that
	\begin{equation}
	\begin{aligned}
		&\int_\BX \norm{\mu_*(x)}_\CG^2 \dd Q_{\scriptscriptstyle{X}}(x) = \EE\big[ \norm{\mu_*(X)}_\CG^2 \big] =  \EE\big[ \norm{\EE [\ell(\cdot, Y) \;|\,X]}_\CG^2 \big] \\[1mm]
		&\leq \EE\big[\EE [\norm{\ell(\cdot, Y)}_\CG^2 \;|\,X] \big] = \EE [ \ell(Y,Y)] < \infty,
	\end{aligned}
	\end{equation}	
	hence $\mu_* \in L_2(\BX, Q_{\scriptscriptstyle{X}};\CG)$.
	\end{proof}
	
	For scalar-valued regression, if $\EE [Y^2] < \infty$ holds, then we have that the conditional expectation is an orthogonal projection to \cl{$L_2(\Omega, \sigma(X), \PP|_{\sigma(X)})$, where $\sigma(X)$ is the $\sigma$-algebra generated by $X$ and $\PP|_{\sigma(X)}$ is the restriction of $\PP$ into $\sigma(X)$}. Similarly:
	\begin{lemma}
\cl{Under $A2$,} $\mu_*(X)$ is an orthogonal projection of $\ell(\cdot,Y)$ to \cl{$L_2(\Omega, \sigma(X), \PP|_{\sigma(X)};\CG)$}.
	\end{lemma}
	\begin{proof}
	For any $\mu \in L_2(\BX, Q_{\scriptscriptstyle{X}}; \CG)$ it holds \cl{\cite[Theorem 4.2]{park2020measure} that}
	\begin{equation}\label{eq:L2error-ort}
	\begin{aligned}
		\mathcal{E} (\mu) &\doteq \EE \big[ \norm{\ell(\cdot, Y) - \mu(X)}_\CG^2 \big] = \EE \big[ \norm{\ell(\cdot, Y) - \mu_*(X) + \mu_*(X) - \mu(X)}_\CG^2 \big]\\
		&= \EE \big[ \norm{\ell(\cdot, Y) - \mu_*(X)}_\CG^2 \big] + \EE \big[ \norm{ \mu_*(X) - \mu(X) }_\CG^2 \big] \\
		& \quad + 2 \,\EE \Big[ \EE \big[ \langle \ell(\cdot, Y) - \mu_*(X),\mu_*(X) - \mu(X)\rangle_\CG\,|\, X \big] \Big]\\
		&= \EE \big[ \norm{\ell(\cdot, Y) - \mu_*(X)}_\CG^2 \big] + \EE \big[ \norm{ \mu_*(X) - \mu(X) }_\CG^2 \big] \\
		& \quad + 2 \, \EE \big[\langle \EE [\ell(\cdot, Y)\,|\,X] - \mu_*(X),\mu_*(X) - \mu(X)\rangle_\CG\big]\\
		&= \EE \big[ \norm{\ell(\cdot, Y) - \mu_*(X)}_\CG^2 \big] + \EE \big[ \norm{ \mu_*(X) - \mu(X) }_\CG^2 \big]. \\
	\end{aligned}
	\end{equation}
	Thus, $\mu_*$ minimizes $\mathcal{E}$ in $L_2(\BX,  Q_{\scriptscriptstyle{X}}; \CG)$.
	\end{proof}
	Note that $\EE \big[ \norm{\ell(\cdot, Y) - \mu_*(X)}_\CG^2 \big]$ is independent of $\mu$, therefore it is reasonable to apply $\EE \big[ \norm{ \mu_*(X) - \mu(X) }_\CG^2 \big]$ as an error criterion. This argument  motivates the following notions of  $L_2(\BX, Q_{\scriptscriptstyle{X}};\CG)$ consistency. 
    Recall that $\QX$ is the marginal distribution of $Q$ w.r.t.\ $X$.
\begin{definition}
		Let $(\mu_n)_{n \in \NN}$ be a $($random$)$ estimate sequence for $\mu_*$. The $($random$)$ $L_2$ error $($\hspace{0.2mm}or more precisely $L_2(\BX, \QX; \CG)$ error\hspace{0.2mm}$)$ of $\mu_n$ is defined by
		\begin{equation}
		\int_\BX \norm{ \mu_*(x) - \mu_n(x)}_\CG^2 \dd \QX(x).
		\end{equation}
		We say that $(\mu_n)_{n \in \NN}$ is weakly consistent for the distribution of $(X,Y)$, if as $n \to \infty$,
		\begin{equation}
			\EE \bigg[\int_\BX \norm{ \mu_*(x) - \mu_n(x)}_\CG^2 \dd \QX(x)\bigg]\! \longrightarrow 0.
		\end{equation}
		We say that $(\mu_n)_{n \in \NN}$ is strongly consistent for the distribution of $(X,Y)$, if as $n \to \infty$,
		\begin{equation}
			\int_\BX \norm{ \mu_*(x) - \mu_n(x)}_\CG^2 \dd \QX(x) \xrightarrow{\,a.s.\,} 0. 
		\end{equation}
	\end{definition}

        One can see that strong consistency does not necessarily imply weak consistency, indeed neither of these two consistency notions imply the other. Nevertheless, knowing that an estimator is weakly consistent will be of great help to also prove its strong consistency.

	Note that it can happen that an estimate sequence is consistent for some distributions of $(X,Y)$ and \cl{inconsistent} for other distributions of $(X,Y)$. We say that $(\mu_n)_{n \in \NN}$ is {\em universally weakly (strongly) consistent} if it is weakly (strongly) consistent for all possible distributions of $(X,Y)$. It is a somewhat surprising fact that there exist such estimates for $\RR^d \to \RR$ type regression functions, see \citep{stone1977consistent}, for example, local averaging adaptive methods such as k-Nearest Neighbors (kNN), several kernel-based methods  and various partitioning rules were proved to be universally consistent, see the book of \cite{gyorfi2002distribution}. 
 
    Recently \cite{hanneke2020universal} proved that there exists a universally strongly consistent classifier (OptiNet) for any separable metric space and \cite{gyorfi2021universal} constructed a universally consistent classification rule called Proto-NN for multiclass classification in metric spaces whenever the 
space admits a universally consistent estimator. For Banach space valued regression in semi-metric spaces \cite{dabo2009kernel} constructed estimates with asymptotic bounds under some differentiation condition on the conditional expectation function and \cite{forzani2012consistent} proved consistency under the Besicovitch condition. Nevertheless, it is still an open question whether there exist universally consistent estimates under more general conditions (e.g., for Hilbert space valued regression).

	\subsection{Generalization of Stone's Theorem}
	
	In this section, we present a generalization of Stone's theorem \citep{stone1977consistent} for conditional kernel mean map estimates for locally compact Polish\footnote{\cl{A Polish space is} a separable, completely metrizable topological space; as the specific choice of the metric is unimportant for the problems studied in this paper, we fix an arbitrary metric $d$ which \cl{is compatible with} the topology of our space, and treat a Polish space as a (complete and separable) metric space.} input spaces.
This theorem will serve as our main theoretical tool to prove consistency results for our new recursive kernel algorithm. We consider local averaging methods of the general form
\begin{equation}\label{genform}
		\mu_n(x) \,\doteq\, \sum_{i=1}^n\, W_{n,i}(x) \,\ell(\cdot, Y_i),
\end{equation}
	for the conditional kernel mean map, $\EE[\, \ell(\cdot, Y)\,|\,X=x\,]$, in a locally compact Polish space $(\BX,d)$. A Polish space gives a convenient setup for probability theory and locally compactness is required because of \cl{Lemma \ref{lemma:Ury}.} We use the fact that the set of compactly supported continuous functions is dense in $L_2(\BX, \QX)$ for every probability measure $\QX$ when $\BX$ is locally compact, see Theorem \ref{denseness}. We deal with data-dependent probability weight functions $W_{n,i}(x)$ for $n \in \NN$ and $i\in [n]$, i.e., we assume that for all $n \in \NN$, weight $W_{n,i}(x)$ is nonnegative for all $i \in [n]$ and $\sum_{i=1}^n W_{n,i}(x) = 1$ for all $x \in \BX$. These are somewhat stronger conditions compare to the original theorem of Stone, however, in practice they are usually satisfied and simplify the proof. For several well-known methods (e.g, kNN, partitioning rules, Nadaraya--Watson estimates) these assumptions hold immediately.
	
		\begin{theorem}\label{thm:stone}
		Let $(\BX, d)$ be a locally compact Polish space and let $\mathcal{X}$ be the Borel $\sigma$-algebra on $\BX$. Assume $A1$ and $A2$.
		Let $\mu_*(x) \,\doteq\, \EE \big[\, \ell(\cdot,Y)\,|\,X=x\,\big]$ and 
        \vspace{1mm}
		\begin{equation}
			\mu_n(x) \,\doteq\, \sum_{i=1}^{n} W_{n,i} (x)\, \ell( \cdot, Y_i).
            \vspace{-1mm}
		\end{equation}
		Assume that:
		\begin{enumerate}[label=(\roman*)]
			\item
			For all $n \in \NN$, $i \in [n]$ and $x \in \BX$, we have
			\begin{equation}\label{eq:cond1}
				\begin{aligned}
					&0 \,\leq \,W_{n,i}(x)\, \leq\, 1, \qquad \text{and} \qquad \sum_{i=1}^n W_{n,i}(x)\, =\, 1. 
				\end{aligned}
			\end{equation}
			\item
			There exists  $c > 0$ such that for all \cl{$f: \BX \to \RR^{+}$ with\, $\EE [f(X)] < \infty$, where $\RR^{+}$ denotes the nonnegative real numbers}, we have
			\begin{equation}\label{eq:cond2}
				\begin{aligned}
					&\EE \bigg[ \sum_{i=1}^n W_{n,i}(X) f(X_i) \bigg] \leq\, c \, \EE [f(X)].
				\end{aligned}
			\end{equation}
			\item
			For all $\delta > 0$, we have
			\begin{equation}\label{eq:cond3}
				\begin{aligned}
					&\lim_{n \to \infty} \EE \bigg[ \sum_{i=1}^n W_{n,i}(X)\, \BI \big( d(X_i, X) > \delta \big)\bigg]=\, 0.
				\end{aligned}
			\end{equation}
			\item
                The following convergence holds for the weight functions:
			\begin{equation}\label{eq:cond4}
				\begin{aligned}
					&\lim_{n \to \infty} \EE \bigg[ \sum_{i=1}^n W_{n,i}^2(X)\bigg] =\, 0.
				\end{aligned}
			\end{equation}			
		\end{enumerate}
		Then, the estimate sequence $(\mu_n)$ is weakly consistent, that is 
		\begin{equation}
			\begin{aligned}
				&\lim_{n \to \infty} \EE\! \int_\BX \norm{\mu_n(x) - \mu_*(x) }_\CG^2 \dd \QX (x) \,=\, 0.
			\end{aligned}
		\end{equation}	
	\end{theorem}
	Condition $(i)$ ensures that \{$W_{n,i}\}_{i=1}^n$ are probability weights for all $n\in \NN$. Assumption $(ii)$ intuitively says that the $L_1$ norm of the local averaging estimate is at most some positive constant times the $L_1$ norm of the estimated function whenever there is no noise on the observations, i.e., when $Y_i = f(X_i)$. Condition $(iii)$ guarantees that asymptotically only those inputs affect the estimate in a point $x$ which are close to $x$. Condition $(iv)$  ensures that the individual weights vanish as the sample size goes to infinity.
    \medskip
 
	\begin{proof}
		The proof is based on the proof of \cite{stone1977consistent} presented in \citep{gyorfi2002distribution}. By first using $(a + b)^2 \leq 2 a^2 + 2 b^2$, we have
        \vspace{-1mm}
			\begin{align}
				&\EE \Big[ \norm{ \mu_n(X) - \mu_*(X)}_\CG^2 \Big]\\[1mm]
				&\leq 2\,\EE \Bigg[ \norm{ \sum_{i=1}^n W_{n,i}(X)\big( \ell(\cdot, Y_i) - \mu_*(X_i)\big)}_\CG^2 \Bigg]
				\!+  2\,\EE \Bigg[ \norm{ \sum_{i=1}^n W_{n,i}(X)\big( \mu_*(X_i) - \mu_*(X)\big)}_\CG^2 \Bigg].\nonumber
			\end{align}
        Then, the application of Lemma \ref{lemma:jn} and Lemma \ref{lemma:in} completes the proof.
        \end{proof}
        \begin{lemma}\label{lemma:jn}
        Under the conditions of Theorem \ref{thm:stone}, we have that, as $n\to \infty$,
        \begin{equation}
            \EE \Bigg[\, \norm{ \sum_{i=1}^n W_{n,i}(X)\big( \mu_*(X_i) - \mu_*(X)\big)}_\CG^2 \Bigg] \to \,0.
        \end{equation}
        \end{lemma}
        \begin{proof}
        Let us use the following notation
        \begin{equation}
            J_n\, \doteq\; \EE \Bigg[\, \norm{ \sum_{i=1}^n W_{n,i}(X)\big( \mu_*(X_i) - \mu_*(X)\big)}_\CG^2 \Bigg].
        \end{equation}
		By Theorem \ref{denseness} there exists a function $\tilde{\mu}_*$ which is bounded and continuous and
		\begin{equation}
			\begin{aligned}
				&\EE \Big[ \norm{\mu_*(X) - \tilde{\mu}_*(X)}_\CG^2 \Big] < \varepsilon.
			\end{aligned}
		\end{equation}
        Recall that by the theorem of Heine and Cantor, see Theorem \ref{thm:heine} in Appendix D, a continuous function with compact support is always uniformly continuous.\
		\cl{Using the convexity of
$f(x) = \norm{x}_\CG$} and  also the elementary fact that $(a + b+ c) ^2 \leq 3 a^2 + 3b^2 +3c^2$ yield
		\begin{equation}
			\begin{aligned}
				J_n &= \EE \Bigg[ \norm{ \sum_{i=1}^n W_{n,i}(X)\big( \mu_*(X_i) - \mu_*(X)\big)}_\CG^2 \Bigg] \\
				&\leq \EE \Bigg[ \bigg( \sum_{i=1}^n W_{n,i}(X)\norm{ \mu_*(X_i) - \mu_*(X) }_\CG \bigg)^2 \Bigg]\\
				&\leq \EE \Bigg[ \sum_{i=1}^n W_{n,i}(X)\norm{ \mu_*(X_i) - \mu_*(X) }_\CG^2 \Bigg]\\
				&\leq 3  \EE \Bigg[ \sum_{i=1}^n W_{n,i}(X)\norm{ \mu_*(X_i) - \tilde{\mu}_*(X_i) }_\CG^2 \Bigg]\\
				&\quad + 3 \EE \Bigg[ \sum_{i=1}^n W_{n,i}(X)\norm{ \tilde{\mu}_*(X_i) - \tilde{\mu}_*(X) }_\CG^2 \Bigg]\\
				&\quad + 3 \EE \Bigg[ \sum_{i=1}^n W_{n,i}(X)\norm{ \tilde{\mu}_*(X) - \mu_*(X) }_\CG^2 \Bigg]	\\[1mm]
				&\doteq 3 \, J_{n}^{(1)} + 3 \, J_{n}^{(2)} + 3 \,J_{n}^{(3)}. 		
			\end{aligned}
		\end{equation}
		We bound these three terms separately. For any $\delta > 0$ one can see that
		\begin{equation}
			\begin{aligned}
				J_{n}^{(2)} &= \EE \Bigg[ \sum_{i=1}^n W_{n,i}(X)\norm{ \tilde{\mu}_*(X_i) - \tilde{\mu}_*(X) }_\CG^2 \BI\big( d(X_i,X) > \delta \big) \Bigg] \\[1mm]
				&\quad + \EE \Bigg[ \sum_{i=1}^n W_{n,i}(X)\norm{ \tilde{\mu}_*(X_i) - \tilde{\mu}_*(X) }_\CG^2 \BI\big( d(X_i,X) \leq  \delta \big) \Bigg]\\
				& \leq  \EE \Bigg[ \sum_{i=1}^n W_{n,i}(X)2\Big( \norm{ \tilde{\mu}_*(X_i)}_\CG^2 + \norm{ \tilde{\mu}_*(X) }_\CG^2 \Big)\BI\big( d(X_i,X) >  \delta \big) \Bigg]\\
				&\quad + \EE \Bigg[ \sum_{i=1}^n W_{n,i}(X)\norm{ \tilde{\mu}_*(X_i) - \tilde{\mu}_*(X) }_\CG^2 \BI\big( d(X_i,X) \leq  \delta \big) \Bigg]\\       
				& \leq 4 \,\sup_{u \in \BX}\norm{\tilde{\mu}(u)}_\CG^2 \,\EE \bigg[\sum_{i=1}^n W_{n,i}(X) \cdot \BI\big( d(X_i,X) > \delta \big) \bigg]\\[1mm]
				&\quad  + \sup_{u,v \in \BX \,:\,d(u,v) \leq \delta } \norm{\tilde{\mu}_*(u) - \tilde{\mu}_*(v) }_\CG^2.\\[1mm]
			\end{aligned}
		\end{equation}
        holds. We know that function $\tilde{\mu}_*$ is uniformly continuous, hence for any $\varepsilon > 0$ one can choose $\delta > 0$ such that
        \vspace{-2mm}
		\begin{equation}
			\sup_{u,v \in \BX \,:\,d(u,v) \leq \delta } \norm{\tilde{\mu}_*(u) - \tilde{\mu}_*(v) }_\CG^2 \;<\, \frac{\varepsilon}{2},
		\end{equation}
        hence the second term is less than $\nicefrac{\varepsilon}{2}$.
		In addition, by condition $(iii)$, for a given $\delta$ if $n$ is large enough, then the first term is also less than $\nicefrac{\varepsilon}{2}$. In conclusion $J_{n}^{(2)}$ converges to $0$\cl{.} 
		
		For the third term we have that
		\begin{equation}
			\begin{aligned}
				|J_{n}^{(3)}|& = \EE \bigg[ \sum_{i=1}^n W_{n,i}(X)\norm{ \tilde{\mu}_*(X) - \mu_*(X) }_\CG^2 \bigg]\\[1mm]
				& = \EE \big[ \norm{\tilde{\mu}_*(X) - \mu_*(X) }_\CG^2\big] < \varepsilon.
			\end{aligned}
		\end{equation}
		For $J_{n}^{(1)}$ the application of condition $(ii)$ to function $f(x) \doteq\norm{ \mu_*(x) - \tilde{\mu}_*(x) }_\CG^2$ yields
        \vspace{-1mm}
		\begin{equation}
			\begin{aligned}
				J_{n}^{(1)} &= \EE \Bigg[ \sum_{i=1}^n W_{n,i}(X)\norm{ \mu_*(X_i) - \tilde{\mu}_*(X_i) }_\CG^2 \Bigg]\\[1mm]
				&\leq c \,\EE \Big[ \norm{ \mu_*(X) - \tilde{\mu}_*(X)}_\CG^2 \Big] =\, c\, \varepsilon.
			\end{aligned}
		\end{equation}
		Therefore, we can conclude that $J_n$ tends to $0$, as $n\to \infty$.
        \end{proof}

        \begin{lemma}\label{lemma:in}
            Under the conditions of Theorem \ref{thm:stone}, we have that
        \begin{equation}
            \EE \Bigg[\, \norm{ \sum_{i=1}^n W_{n,i}(X)\big( \ell(\cdot, Y_i) - \mu_*(X_i)\big)}_\CG^2\Bigg] \to 0.
        \end{equation}
        \end{lemma}
        \begin{proof}
        For simplicity let
        \begin{equation}
			\begin{aligned}
				I_n &\doteq \EE \Bigg[ \norm{ \sum_{i=1}^n W_{n,i}(X)\big( \ell(\cdot, Y_i) - \mu_*(X_i)\big)}_\CG^2\Bigg]\\[1mm]
                &=\EE \bigg[ \sum_{i=1}^{n} \sum_{j=1}^n W_{n,i} (X) W_{n,j}(X) \, \langle \ell(\cdot, Y_i)- \mu_*(X_i), \ell(\cdot,Y_j)- \mu_*(X_j) \rangle_\CG \bigg].
			\end{aligned}
		\end{equation}
		For $i\neq j$ we have that
		\begin{equation}\label{eq:ortog}
			\begin{aligned}
				&\EE \bigg[W_{n,i} (X) W_{n,j}(X) \langle \ell(\cdot, Y_i) - \mu_*(X_i), \ell(\cdot,Y_j)- \mu_*(X_j) \rangle_\CG \bigg]\\
				& = \EE \bigg[\EE \Big[W_{n,i} (X) W_{n,j}(X) \,\langle \ell(\cdot, Y_i)- \mu_*(X_i), \ell(\cdot,Y_j)- \mu_*(X_j) \rangle_\CG \,|\, X, X_1, \dots X_n, Y_i\Big] \bigg]\\
				& = \EE \bigg[W_{n,i} (X) W_{n,j}(X)\, \big\langle \ell(\cdot, Y_i)- \mu_*(X_i), \EE \Big[\ell(\cdot,Y_j)\,|\, X, X_1, \dots X_n, Y_i\Big] - \mu_*(X_j)\big\rangle_\CG \bigg]\\
				& = \EE \bigg[W_{n,i} (X) W_{n,j}(X) \,  \big\langle \ell(\cdot, Y_i)- \mu_*(X_i), \EE \big[\ell(\cdot,Y_j)\,|\,X_j\big] - \mu_*(X_j)\big\rangle_\CG \bigg]= 0.
			\end{aligned}
		\end{equation}
		Furthermore, by $\EE [\ell(Y,Y)] < \infty$ we have
		\begin{equation}
			\begin{aligned}
				&\EE \big[ \sigma^2(X)\big] < \infty
			\end{aligned}
		\end{equation}
		for $\sigma^2(X) \doteq \EE \big[ \norm{\ell(\cdot, Y) - \mu_*(X) }_\CG^2 \,|\, X\big]$. By \eqref{eq:ortog} we have
		\begin{equation}
			\begin{aligned}
				I_n &= \EE \bigg[\sum_{i=1}^n W_{n,i}^2 (X) \norm{\ell(\cdot,Y_i)- \mu_*(X_i)}_\CG^2 \bigg] \\
				&= \EE \bigg[\EE\Big[\sum_{i=1}^n W_{n,i}^2 (X) \norm{\ell(\cdot,Y_i)- \mu_*(X_i)}_\CG^2\,\Big| X,X_1, \dots, X_n\Big] \bigg]\\
				&= \EE \bigg[\sum_{i=1}^n W_{n,i}^2 (X) \EE\Big[ \norm{\ell(\cdot,Y_i)- \mu_*(X_i)}_\CG^2\,\Big| X,X_1, \dots, X_n\Big] \bigg]\\
				&= \EE \bigg[\sum_{i=1}^n W_{n,i}^2 (X) \EE\Big[ \norm{\ell(\cdot,Y_i)- \mu_*(X_i)}_\CG^2\,\Big| X_i\Big] \bigg]\\
				&= \EE \Big[\sum_{i=1}^n W_{n,i}^2 (X) \sigma^2(X_i)\Big].
			\end{aligned}
		\end{equation}	
		If $\sigma^2$ is bounded, then $I_n$ tends to zero almost surely by condition $(iv)$. Otherwise let $\tilde{\sigma}^2 \leq L$ be such that 
        \vspace{-2mm}
		\begin{equation}
			\begin{aligned}
				\EE \big[ | \tilde{\sigma}^2 (X) - \sigma^2(X) |\big] < \varepsilon,
			\end{aligned}
		\end{equation}
		see \cite[Theorem A.1]{gyorfi2002distribution}. Then
		\begin{equation}
			\begin{aligned}
				I_n &\leq \EE \Big[ \sum_{i=1}^n W_{n,i}^2(X) \tilde{\sigma}^2(X_i)\big] + \EE \Big[ \sum_{i=1}^n W_{n,i}^2(X)\, \big|\, \sigma^2(X_i) - \tilde{\sigma}^2(X_i)\big|\Big] \\
				&\leq L \, \EE \Big[\, \cl{\sum_{i=1}^n}\, W_{n,i}^2(X)\Big] + \EE \Big[ \sum_{i=1}^n W_{n,i}(X)\, \big|\, \sigma^2(X_i) - \tilde{\sigma}^2(X_i)\big|\Big],
			\end{aligned}
		\end{equation}
		where the first term tends to zero by condition $(iv)$ and for the second term we have
		\begin{equation}
			\begin{aligned}
				&\EE \Big[ \sum_{i=1}^n W_{n,i} (X)\big|\, \sigma^2(X_i) - \tilde{\sigma}^2(X_i)\big|\Big] \leq c \, \EE\,\big[\,\big|\,\sigma^2(X) - \tilde{\sigma}^2(X)\,\big|\,\big] \leq c \, \varepsilon
			\end{aligned}
		\end{equation}
	by condition $(ii)$. In conclusion, $I_n$ converges to $0$.
	\end{proof}

	\subsection{Recursive Estimation of Conditional Kernel Mean Embeddings}
	
	Let $(\BX, d)$ be a locally compact Polish space, as before. In this section, we apply local averaging {\em smoothers} or ``kernels'' on the input space \citep{gyorfi2002distribution}. These ``kernels'' are not necessarily positive definite, therefore, to avoid misunderstanding, we call them smoothers. Nevertheless, many positive definite functions are included in the definition.
	\begin{definition}[smoother]
		A function $k: \BX \times \BX \to \RR$ is called a smoother, if it is measurable, non-negative, symmetric, and bounded.
	\end{definition} 
 
 	Now, we introduce a recursive local averaging estimator.\ Let $(k_n)_{n\in \NN}$ be a smoother sequence and $(a_n)_{n\in\NN}$ be a stepsize (learning rate or gain) sequence of positive numbers.
  
    Let us consider the following recursive algorithm:
    \vspace{-0.5mm}
	\begin{equation}\label{eq:recursion}
		\begin{aligned}
			\mu_1(x) &\,\doteq\, \ell(\cdot, Y_1), \quad \text{and}\\[1.5mm]
			\mu_{n+1}(x) &\,\doteq\hspace{0.2mm} \big(1 - a_{n+1} k_{n+1}(x, X_{n+1})\big) \mu_n(x) \,+\, a_{n+1} k_{n+1}(x,X_{n+1}) \ell(\cdot,Y_{n+1}),
		\end{aligned}
	\end{equation}
	for $n \geq 1$. By induction, one can see that $\mu_n$ is of the form defined in \eqref{genform}. It is clear that one only needs to store $\mu_n$, because in each iteration, when a new observation pair is available, one can immediately update $\mu_n$ by the recursion defined in \eqref{eq:recursion}. 

    \cl{In practice, we often need to evaluate $\mu_*(x)$ on a function $g \in \CG$, i.e., the conditional expectation $\EE [\hspace{0.8mm}g(Y)\,|\, X= x\hspace{0.4mm}]$ is sought. Then, our recursion simplifies to
    \begin{equation}
    \begin{aligned}
        &\langle \mu_1(x), g \rangle_\CG = \langle \ell(\cdot, Y_1), g \rangle_\CG = g(Y_1),\\[1.5mm]
        &\langle \mu_{n+1}(x), g \rangle_\CG =  \langle \big(1 - a_{n+1} k_{n+1}(x, X_{n+1})\big) \mu_n(x) \,+\, a_{n+1} k_{n+1}(x,X_{n+1}) \ell(\cdot,Y_{n+1}), g \rangle_\CG\\[1mm]
        & =  \big(1 - a_{n+1} k_{n+1}(x, X_{n+1})\big) \langle \mu_n(x), g\rangle_\CG + a_{n+1} k_{n+1}(x,X_{n+1}) g(Y_{n+1}),
    \end{aligned}
    \end{equation}
    which can be computed in linear time with constant memory. The real-valued consistency of this method can be established, e.g., by \cite[Theorem 25.1]{gyorfi2002distribution}, however, we usually do not know in advance which function $g \in \CG$ is of our interest, hence 
    estimating the CMKE is very useful.
Moreover, knowing the whole conditional distribution could also be used for various inference tasks, e.g., 
    hypothesis testing or building prediction regions.}

    \cl{Theorem} \ref{thm:sajat} presents sufficient conditions for the weak and strong consistency of our recursive scheme. It is a generalization of Theorem 25.1 of \citet{gyorfi2002distribution}. There are two novelties w.r.t.\ the result in \citep{gyorfi2002distribution}.\ We prove consistency for {\em Hilbert space valued regression} and we cover {\em locally compact Polish input spaces}.\ The smoother functions are controlled on an intuitive and flexible manner. They do not need to admit a particular form or be monotonous. Intuitively the smoother functions should shrink around the observed input points as $n$ goes to infinity. The rate of this decrease is controlled by stepsizes $(h_n)$ and $(r_n)$. The proposed conditions are w.r.t.\ a general majorant function $H$.
	
	\begin{theorem}\label{thm:sajat}
		Let $(\BX,d)$ be a locally compact Polish space. Assume that $A1$, $A2$ and the following conditions are satisfied.
		\begin{enumerate}[label=(\roman*)]
                \item 
                There exist sequences $(h_n)_{n\in \NN}$ and $(r_n)_{n\in \NN}$ of positive numbers tending to $0$ and a non-negative, nonincreasing function $H$ on $[0, \infty)$ with $1/r_n \cdot H(s/h_n) \to 0$ $(n \to \infty)$ for all $s \in (0,\infty)$.
			\item For all $n \in \NN$ it holds that
            \vspace{-1mm}
			\begin{equation}
				\begin{aligned}
				k_n (x,z) \,\leq\, \frac{1}{r_n}H\big(\, d(x,z)/h_n\,\big).
				\end{aligned}
			\end{equation}
			\item
                The supremum of additive weights is bounded by $1$, i.e.,
			\begin{equation}
				\sup_{x,z \in \BX, n \in \NN} a_n k_n(x,z) \,\leq \,1.
			\end{equation}
			\item
                Measure $\QX$ does not diminish w.r.t.\ the smoother functions, that is
			\begin{equation}\label{eq:liminf0}
				\liminf_{n \to \infty} \int_{\cl{\BX}} k_n(x,t) \dd \QX(t)\, >\,0  \quad \text{for\, $\QX$-almost all $x$.}
			\end{equation}
			\item
                The \cl{positive} learning rates satisfy
			\begin{equation}
			\begin{aligned}
				\sum_{n =1}^\infty a_n\, =\, \infty, \qquad \text{ and }\qquad \sum_{n =1}^\infty \frac{a_n^2}{r_n^{2}}\,<\, \infty.
			\end{aligned}			
			\end{equation}
		\end{enumerate}
		Then, the estimate sequence $(\mu_{n})_{n\in \NN}$ defined by \eqref{eq:recursion} is weakly and strongly consistent.
	\end{theorem}
	
	Conditions $(i)$ and $(ii)$ ensure that $k_n(x,z)$ vanishes for $x\neq z$ \cl{as $n \to \infty$}. The weight functions are probability weights because of assumption $(iii)$. Condition $(iv)$ is the strongest technical assumption in a general metric measure space, however, in  several cases it is easy to verify. Condition $(v)$ can be usually guaranteed by the choice of stepsizes. 

    \medskip
	
	\begin{proof}
		For the weak consistency, we apply Theorem \ref{thm:stone}, hence our main task is to verify its assumptions. Observe that with $W_{1,1}(x) = 1$ and
		\begin{equation}
			W_{n,i}(x)\, = \prod_{l=i+1}^n \big(1- a_l k_l(x,X_l)\big) \,a_i k_i(x,X_i),
		\end{equation}
		for $n \geq 2$ and $i \in [n]$ we have $\mu_n(x) = \sum_{i=1}^n W_{n,i}(x) \ell(\cdot, Y_i)$. By induction, it is easy to see that the weights are probability weights. We start the procedure with a constant $1$ function, then an update clearly does not change the sum of the weights for any input point $x \in \BX$. 
  
        Condition \eqref{eq:cond2} is verified in the proof of  \citet[Theorem 25.1]{gyorfi2002distribution}. Note that \eqref{eq:cond2} only regards the weight functions, which are identical for real-valued regression. Condition \eqref{eq:cond4} again follows from the scalar-valued case.
		
		In order to prove that condition \eqref{eq:cond3} holds, we follow the proof in \cite[Theorem 25.1]{gyorfi2002distribution} with necessary modifications. By using independence of $\{X_i\}_{i=1}^n$ and the inequality $(1-x) \leq e^{-x}$, we have for all $i \leq n$ that
        \vspace{-1mm}
		\begin{equation}\label{eq:exptozero}
		\begin{aligned}
			\EE \big[ \,W_{n,i}(x)\,\big] &= \prod_{l=i+1}^n\EE\big[\, (1- a_l k_l(x,X_l))\,\big]\, \EE\big[\,a_i k_i(x,X_i)\,\big]\\[1mm]
			&\leq \frac{a_i \, H(0)}{r_i}e^{-\sum_{l=i+1}^n  a_l \EE \big[ k_l(x,X_l)\big]} \xrightarrow{n \to \infty} 0,
		\end{aligned}
		\end{equation}
		because $\sum_{l=i+1}^n  a_l \EE \big[ k_l(x,X_l)\big] \xrightarrow{n \to \infty} \infty$ by condition $(iv) $ and $\sum_{n=1}^\infty a_n = \infty$. Let us fix $\delta > 	0$. By the dominated convergence theorem, it is sufficient to prove that for almost all $x$ it holds that
        \vspace{-1mm}
		\begin{equation}
		\begin{aligned}
			\EE \bigg[ \sum_{i=1}^n W_{n,i}(x) \BI\big(d(x,X_i) > \delta\big)\bigg] \to 0. 
		\end{aligned}
		\end{equation}
		Let $p_n(x)\doteq \EE [k_n(x,X)]$. By condition $(iv)$ we have that
		\begin{equation}
			\begin{aligned}
				p(x)\, \doteq\, \frac{1}{2} \liminf_{n \to \infty} p_n(x)\, >\, 0
			\end{aligned}
		\end{equation}
		holds $\QX$-almost everywhere implying that for $\QX$-almost all $x$ there exists $n_0(x)$ such that, for $n > n_0(x)$ we have $p_n(x) \geq p(x)$. Because of \eqref{eq:exptozero}, for such $x \in \BX$ it is sufficient to prove
		 \begin{equation}
		 	\begin{aligned}
		 		\EE \bigg[ \sum_{i = n_0(x)}^n W_{n,i}(x) \BI \big( d(x,X_i) > \delta \big) \bigg]\xrightarrow{n \to \infty} 0.
		 	\end{aligned}
		 \end{equation}
	     By independence one can observe that
      \vspace{-1mm}
	     \begin{equation}
	     	\begin{aligned}
	     		&\EE \Big[  W_{n,i}(x) \BI \big( d(x,X_i) > \delta \big) \Big] = \EE \bigg[ \prod_{l=i+1}^n (1- a_lk_l(x,X_l) ) a_i k_i(x,X_i) \BI\big( d(x,X_i) > \delta \big)\bigg]\\
	     		&= \EE \bigg[ \prod_{l=i+1}^n (1- a_lk_l(x,X_l) )\bigg] \EE \big[ a_i k_i(x,X_i) \BI\big( d(x,X_i) > \delta \big)\big]\\
	     		&= \frac{\EE \big[  W_{n,i}(x)\big]}{\EE \big[ a_i k_i(x,X_i) \big]}\EE \big[ a_i k_i(x,X_i) \BI\big( d(x,X_i) > \delta \big)\big],
	     	\end{aligned}
	     \end{equation}
     	therefore, we have
     	\begin{equation}
     		\begin{aligned}
     			&\EE \bigg[ \sum_{i=n_0(x)}^n W_{n,i}(x) \BI \big( d(x,X_i) > \delta \big) \bigg]\\
     			&= \sum_{i=n_0(x)}^n \EE \big[  W_{n,i}(x)\big] \frac{\EE \big[ a_i k_i(x,X_i) \BI\big( d(x,X_i) > \delta \big)\big]}{\EE \big[ a_i k_i(x,X_i) \big]}\\
     			& \leq \sum_{i=n_0(x)}^n \EE \big[  W_{n,i}(x)\big] \frac{ H(\delta/h_i)}{r_i \, p(x)}. 
     		\end{aligned}
     	\end{equation}
     	Since $\sum_{i=n_0(x)}^n \EE \big[  W_{n,i}(x)\big] \leq 1$, $\EE \big[  W_{n,i}(x)\big] \to 0$ and $\nicefrac{1}{r_n}\,  H(\delta/h_n) \to 0$ as $n\to \infty$ the application of Lemma \ref{lemma:Toeplitz} yields
     	\begin{equation}
     		\begin{aligned}
     			\sum_{i=n_0(x)}^n \EE \big[\,  W_{n,i}(x)\,\big]\, \frac{ H(\delta/h_i)}{r_i \,p(x)} \to\, 0. 
     		\end{aligned}
     	\end{equation}
     	Consequently, the weak consistency of $(\mu_n)$ follows from Theorem \ref{thm:stone}\hspace*{0.25mm}.
     	
     	The strong consistency of $(\mu_n)_{n\in \NN}$ can be proved based on the almost supermartingale convergence theorem of \cite{robbins1971convergence}, see Theorem \ref{thm:robbins-ziegmund}\hspace*{0.25mm}. The technical details of the proof are presented in Appendix C.
	\end{proof}

        We presented our recursive algorithm in a general setup. It remains to construct smoother sequences and stepsizes that satisfy the proposed conditions. 
        In the theorem that follows we generate the smoothers with the help of an $\RR \to \RR$ type {\em mother kernel} $K$. In the next section we present several corollaries of this theorem to design consistent algorithms for many important machine learning applications.
        
	\begin{theorem}\label{thm:result}
		Let $(\BX,d)$ be a locally compact Polish space and $K: \RR \to \RR$ be a measurable, nonnegative and bounded function. Assume $A1$ and $A2$. Let $(h_n)_{n \in \NN}$ and $(r_n)_{n \in \NN}$ be positive sequences with zero limit. Let us define the following smoothers 
		\begin{equation}\label{eq:smoother}			
		\begin{aligned}
			k_n(x,z) \,\doteq\, \frac{1}{r_n} K \bigg( \frac{d(x,z)}{h_n}\bigg),
		\end{aligned}
		\end{equation} 
		for $n \in \NN$. Let $(\mu_n)_{n\in \NN}$ be defined as in \eqref{eq:recursion}. Assume that:
		\begin{enumerate}[label=(\roman*)]
			\item There are positive constants $R$ and $b$ such that $K(t)\geq b \, \BI( t \in B(0,R) )$ for all $t \in \RR$, \cl{where $B(0,R)$ denotes the open ball with center $0$ and radius $R$}.
			\item There is a measurable, nonnegative and nonincreasing function $H$ on $[0,\infty)$ such that $1/r_n\, H(s/h_n) \to 0$ as $n\to \infty$ and $K(s) \leq H(s)$ for all $s \in (0,\infty)$.
			\item The nonnegative learning rates satisfy
			\begin{equation}\label{eq:condiii}
				\sum_{n=1}^\infty a_n \,=\,\infty, \qquad \text{ and }\qquad \sum_{n=1}^\infty \frac{a_n^2}{r_n^2} \,<\,\infty.
			\end{equation}
			\item
                The weights of the update term is bounded by $1$, i.e.,
			\begin{equation}\label{eq:condiv}
				\sup_{t \in \RR} K(t) \,\sup_{n\in \NN} \frac{a_n}{r_n} \;\cl{\leq}\; 1.
			\end{equation}
			\item
                The probability of small balls is bounded from below by
			\begin{equation}\label{eq:condv}
				\liminf_{n \to \infty}\frac{\QX( B(x,R \,h_n))}{r_n}\,>\,0 \qquad\QX\text{\,-\,a.s.}
			\end{equation}
		\end{enumerate}
		Then, $(\mu_n)_{n\in \NN}$ is weakly and strongly consistent.
	\end{theorem}
	\begin{proof}
		It is sufficient to verify the conditions of Theorem \ref{thm:sajat}. We assumed that
        \begin{equation}
            \nicefrac{1}{r_n}\, H(s/h_n) \to 0,
        \end{equation}
        for all $s \in (0,\infty)$. From \eqref{eq:smoother} and condition $(ii)$ it follows that
        \begin{equation}
            r_n k_n(x,z) = K(d(x,z)/h_n)\leq H(d(x,z)/h_n).
        \end{equation}
        Condition $(iv)$ implies that
		\begin{equation}
			\sup_{x,z \in \BX, n \in \NN} a_n k_n(x,z)= \sup_{x,z \in \BX, n \in \NN} \frac{a_n}{r_n} K\bigg(\frac{d(x,z)}{h_n}\bigg)\leq \sup_{t \in \RR} K(t) \sup_{n\in \NN} \frac{a_n}{r_n} \,\cl{\leq} \,1
		\end{equation}
        holds.
		Finally by condition $(i)$ we have
		\begin{equation}
		\begin{aligned}
			&\int_\BX k_n(x,t) \dd \QX(t) \cl{\;=}	\int_\BX \frac{1}{r_n}K\bigg( \frac{d(x,t)}{h_n}\bigg) \dd \QX(t)\\
			&\geq \int_\BX \frac{b}{r_n}\BI\bigg( \frac{d(x,t)}{h_n} \in B(0,R)\bigg) \dd \QX(t)\\
			&\cl{=}  \int_\BX \frac{b}{r_n}\BI\bigg( t \in B(x,R\, h_n)\bigg) \dd \QX(t) = \frac{b \, \QX(B(x, R\, h_n))}{r_n}.
		\end{aligned}
		\end{equation}
		Taking the $\liminf$ on both sides and applying condition $(v)$ yield \eqref{eq:liminf0}.
		The conditions on the stepsize sequence are assumed. Therefore, by applying Theorem \ref{thm:sajat}, we can conclude that the estimate sequence, $(\mu_n)_{n \in \NN}$, is weakly and strongly consistent.
	\end{proof}
	\section{Demonstrative Applications}

        In this section we demonstrate the applicability of Theorem \ref{thm:result} through some characteristic model classes important for statistical learning. We consider three general setups. The main difference of these applications lies in the input spaces. We show that the requirement on the probability of small balls is satisfied in these cases and provide constructions for the learning rates, $(a_n)$, $(r_n)$ and $(h_n)$, such that they satisfy the conditions of Theorem \ref{thm:result}.
        
        First, {\em Euclidean} input spaces are considered, for which we prove the weak and strong universal consistency of the recursive estimator of the conditional kernel mean map. The main novelty of this corollary compared to the results of \citet[Chapter 25]{gyorfi2002distribution} is that our method is applicable for special {\em Hilbert space valued regression} problems.
        
        Second, we study abstract {\em Riemannian manifolds} as input spaces for which an embedding into an ``ambient'' Euclidean space may not be explicitly given.
One of the motivations of this example comes from {\em manifold learning} \citep{tenenbaum2000global,lin2008riemannian}, which is an actively studied area in machine learning, whose theoretical study is full of  challenges. The presented recursive framework can offer a new technical tool for this direction. 

        Finally, we investigate the case of {\em functional inputs} \citep{ferraty2006nonparametric}. This example is motivated by {\em signal processing} problems and methods for continuous-time systems in time-series analysis. We consider a wavelet-like approach leading to an input space that is a locally compact subset of $L_2(\RR)$ induced by translations of a suitable {\em mother function}.

	\subsection{Euclidean Spaces}
	When $\BX$ is an Euclidean space, the weak and strong universal consistency of the recursive estimator of the conditional kernel mean map is a consequence of Theorem \ref{thm:result}.
	
	\begin{corollary}
		Let $\BX = \RR^p$ be a $p$-dimensional Euclidean space, with the Euclidean metric, and $K: \RR \to \RR$ be a measurable, nonnegative and bounded function. Let the stepsizes be defined as $a_n \doteq \nicefrac{1}{n}$, and let $r_n\doteq \nicefrac{1}{n^{(1-\varepsilon)/2}}$, where $\varepsilon \in (0,1)$, and $h_n \doteq \sqrt[p]{r_n}$ and $k_n$ be as in \eqref{eq:smoother}
		for $n \in \NN$. Let $(\mu_n)_{n\in \NN}$ be defined as in \eqref{eq:recursion}. Assume A1, A2 and conditions $(i)$ and $(ii)$ from Theorem \ref{thm:result}, then $(\mu_n)_{n \in \NN}$ is weakly and strongly universally consistent.
	\end{corollary}

	\begin{proof}
        We can assume (w.l.o.g.) that $K$ is bounded by $1$. A Euclidean space is a locally compact Polish space, therefore, we can apply Theorem \ref{thm:result}. Conditions $(i)$ and $(ii)$ are assumed. Our choice of stepsizes clearly satisfies \eqref{eq:condiii} and \eqref{eq:condiv}. For \eqref{eq:condv} we use \cite[Lemma 2.2]{devroye1981almost} which proves that for every probability measure $\QX$ on $\RR^p$ there exists a nonnegative function $g$ satisfying $g(x) < \infty$ $\QX$-almost surely, such that 
		\begin{equation}
		\begin{aligned}
			\frac{h_n^p}{\QX(B(x,R\, h_n))} \to g(x) \qquad \QX-\text{a.s.}
		\end{aligned}
		\end{equation}
		for all $R > 0$. Therefore
		\begin{equation}
			\begin{aligned}
				\liminf_{n \to \infty} \frac{\QX(B(x,R\, h_n))}{r_n} = \frac{1}{g(x)} > 0\qquad \QX-\text{a.s.}
			\end{aligned}
		\end{equation}
        and then the statement of corollary is proved by applying Theorem \ref{thm:result}.
	\end{proof}
	
	\begin{remark}
Mother kernel $K$ can be chosen based on prior knowledge on the problem. In the table that follows we list some of the possible candidates.
\begin{center}
			\begin{tabular}{ |l |l |}
				\hline
				\text{Name} & \text{Formula}\\
				\hline
				\hline
				Box kernel &$\BI(|x| < B)\Big.$, with $B > 0$\\
				Gaussian kernel & $\exp\Big(- \frac{x^2}{\sigma}\Big)$, with $\sigma > 0$ \\  
				Laplace kernel &      $\exp\Big(- \frac{|x|}{\sigma}\Big)$, with $\sigma > 0$\\[2mm]
				Epanechnikov kernel & $3/4 (1- x^2) \, \BI(|x| < 1)$\\[1mm]
				\hline
			\end{tabular}
		\end{center}
		One can easily verify that all of these functions are measurable, nonnegative and bounded by $1$. They also satisfy conditions $(i)$ and $(ii)$ from Theorem \ref{thm:result} with $H = K|_{[0,\infty)}$.
	\end{remark}
	
	\subsection{Riemannian Manifolds}

	    In our recursive scheme the power of $h_n$ w.r.t.\ $r_n$ is (at least) $p$ for a $p$-dimensional Euclidean space. Manifolds can often be embedded in a high-dimensional Euclidean space, however, they are locally isomorphic to a lower dimensional one. This local isometry allows us to use the lower manifold dimension as the power of $h_n$ w.r.t.\ $r_n$. Moreover, we also cover {\em abstract} $p$-dimensional manifolds, hence an ambient Euclidean space is not required.
	\begin{corollary}\label{cor:riemannian}
		Let $\BX$ be a $p$-dimensional connected, geodesically complete Riemannian manifold, with geodesical distance $\varrho$ and $K: \RR \to \RR$ be a measurable, nonnegative and bounded function. Let $a_n \doteq \nicefrac{1}{n}$, $r_n\doteq \nicefrac{1}{n^{(1-\varepsilon)/2}}$, where $\varepsilon \in (0,1)$, and $h_n = \sqrt[p]{r_n}$ and $k_n$ be as in \eqref{eq:smoother}
		for $n \in \NN$. Let $(\mu_n)_{n \in \NN}$ be defined as in \eqref{eq:recursion}.  Assume A1, A2 and conditions $(i)$ and $(ii)$ from Theorem \ref{thm:result}, then $(\mu_n)_{n \in \NN}$ is weakly and strongly universally consistent.
	\end{corollary}
    
	\begin{proof}
    Assume (w.l.o.g.) that $K$ is bounded by $1$. 
	Every topological manifold is locally compact \citep[Proposition 1.12]{lee2013smooth}. Connected Riemannian manifolds are metrizable, henceforth paracompact\footnote{\cl{A topological space is paracompact, if every open cover has an open refinement that is locally finite.}} \citep{stone1948paracompactness}. A connected paracompact metric space is second countable, thus it is separable and admits a countable atlas. Moreover, by the Hopf-Rinow theorem a geodesically complete Riemannian manifold is complete as a metric space.

Consider a countable smooth atlas $\CA$ on $\BX$. We are going to prove the assumption about small ball probabilities for every chart in the atlas almost surely, hence the claim for the whole space will follow almost surely. Consider a chart $(U, f)$ in $\CA$. 
    
    If $\QX(U) =0$, then the condition is satisfied automatically, otherwise consider the probability measure $\QX/\QX(U)$ on $U$. We push measure $\QX/\QX(U)$ to $\RR^p$, i.e., consider the probability measure $\mu$ on $\RR^p$, defined by $\mu(A) \doteq \QX \circ f^{-1} [f(U) \cap A]/\QX(U)$ for every Borel set $A$ in $\RR^p$. We have already seen that 
    \begin{equation}\label{eq:liminf}
		\begin{aligned}
			\liminf_{n \to \infty} \frac{\mu(B_2(s,R\, h_n))}{h_n^p} > 0\qquad \mu-\text{almost everywhere},
		\end{aligned}
	\end{equation}
    that is there exists a measurable set $\Omega_1 \subseteq \RR^p$ such that $\mu(\Omega_1)= 1$ and for all $s \in \Omega_1$ one has \eqref{eq:liminf}. Notice that one can assume that $\Omega_1 \subseteq f(U)$ (otherwise we can take the intersection).
   
    Let $q \in \Omega_1$. We know that $f^{-1}$ is a smooth diffeomorphism, therefore there exists $C >0$ such that $\norm{Df^{-1}(s)} \leq C$ for $s \in B_2(q, 1/C \, h_n)$\cl{, where $D$ is the differential operator and $\norm{\cdot}$ denotes the operator norm}. In the next step, we prove that $B_\varrho(f^{-1}(q), h_n) \supseteq f^{-1}[B_2(q,1/C \, h_n)\cap f(U)]$ for $h_n$ small enough.
    Let $s  \in B_2(q,1/C \, h_n)\cap f(U)$. Then, there are 
    $x = f^{-1}(q)$ and $\bar{x} \doteq f^{-1}(s)$. By Lemma \ref{meanvalue-R}, we have
    \begin{equation}
    \begin{aligned}
        \varrho(x, \bar{x}) = \varrho(f^{-1}(q), f^{-1}(s)) \leq C \, \norm{q- s}_2 \leq h_n.
    \end{aligned}
    \end{equation}
    It follows that
    \begin{equation}\label{eq:67}
        \QX(B_\varrho(f^{-1}(q),h_n)) \geq \QX\circ f^{-1}(B_2(q,1/C \, h_n)\cap f(U)) =\mu(B_2(q, 1/C\, h_n))\, \QX(U).
    \end{equation}
    Let $\Omega_0 = f^{-1}[\Omega_1] \subseteq \BX$, for which $\QX(\Omega_0)/\QX(U)=1$. Moreover, for $x \in \Omega_0$ we have that 
    \begin{equation}\label{eq:liminf2}
        	\liminf_{n \to \infty} \frac{\QX(B_\varrho(x,h_n))}{h_n^d} \geq \liminf_{n\to \infty} \frac{\QX(U)\mu(B_2(f(x),1/C \, h_n))}{h_n^d}> 0.
    \end{equation}
    holds for every $x \in \Omega_0$. Hence \eqref{eq:liminf2} holds $\QX/\QX(U)$-almost surely and $\QX$-almost everywhere.
    The rest of the proof is similar to that of the Euclidean case.
\end{proof}
	
	\subsection{Locally Compact Subset of a Hilbert Space}

        Now, we consider {\em functional inputs}, which can be useful, e.g.,
in {\em signal processing}. The analyzed function set is motivated by {\em wavelets}. Our model class can be seen as how a signal is encoded in a computer by storing a finite number of coefficients for some basis functions at given knots, where the locations of these knots are flexible and part of the representation.
        
        Let us consider the Hilbert space $L_{2}(\RR )$. Let $\psi$ be a function in a Sobolev space $W^{1,2}(\RR)$ \cl{\citep{berlinet2004reproducing}}, and $\tau_x$ be the translation operator on $L_{2}(\RR)$ defined pointwise by $\tau_x f(t) \doteq f(t- x)$. 
\cl{For given (constant) hyper-parameters $m \in \NN, M > 0,$ let}
        \vspace{-1mm}
	\begin{equation}
	\begin{aligned}
		\mathcal{M}\, \doteq\, \bigg\{\,f:\RR \to \RR \;\,\Big|\;\, f = \sum_{i=1}^m \lambda_i \, \tau_{x_i}\psi \text{ for some }\lambda_1, \dots, \lambda_m, x_1, \dots , x_m \in [-M,M] \,\bigg\}.
	\end{aligned}
    \vspace{-1mm}
	\end{equation}
	Consider the canonical (measurable) mapping $\phi$ from $[-M,M]^{2m}$ to $\mathcal{M}$  given by
	\begin{equation}
    \label{eq:map-phi}
		\begin{aligned}
            \phi:
			\begin{bmatrix}
				\lambda\\
				x
			\end{bmatrix}
			\mapsto \sum_{i=1}^m \lambda_i \, \tau_{x_i}\psi, 	
		\end{aligned}
	\end{equation}
	where $\lambda = [\lambda_1, \dots , \lambda_m]\tr$ and $x = [x_1, \dots, x_m]\tr$. Clearly, $\phi$ is surjective, however, it is not injective. We are going to consider probability measures of the form $\QX \doteq \PX \circ \phi^{-1}$ for some $\PX$ on $[-M,M]^{2m}$, i.e., we assume that $\QX$ is the pushforward of some measure on the bounded (locally compact) Polish space $[-M,M]^{2m}$ with the Euclidean metric.

	\begin{corollary}
		Let \cl{$\psi \in W^{1,2}(\RR)$}, $\BX=\mathcal{M}$ with the $L_2$ metric, $K: \RR \to \RR$ be a measurable, nonnegative and bounded function. Let $a_n \doteq \nicefrac{1}{n}$, $r_n\doteq \nicefrac{1}{n^{(1-\varepsilon)/2}}$, for an $\varepsilon \in (0,1)$, and $h_n = \sqrt[2m]{r_n}$ and $k_n$ be as in \eqref{eq:smoother}
		for $n \in \NN$. Let $(\mu_n)_{n\in \NN}$ be defined as in \eqref{eq:recursion}. Assume $(i)$-$(ii)$ from Theorem \ref{thm:result}, then $(\mu_n)_{n\in \NN}$ is weakly and strongly  consistent for measures of the form $\QX \doteq \PX \circ \phi^{-1}$ for any distribution $\PX$ on $[-M,M]^{2m}$, where $\phi$ is given by \eqref{eq:map-phi}.
	\end{corollary}

	\begin{proof}
		We need to prove that $\mathcal{M}$ is locally compact and Polish. 
		First, we show that $\mathcal{M}$ is complete with the $L_2$ metric. It is known that $L_2$ is complete. Hence, for any Cauchy sequence $\cl{(f_n)_{n \in \NN}}$ in $\mathcal{M}$ we have a limit function $f \in L_2(\RR)$. We show that $f$ is also in $\mathcal{M}$. Let $\cl{f_n = \sum_{i=1}^m \lambda_{n,i} \, \tau_{x_{n,i}} \psi}$. We know that $[-M,M]$ is a compact set hence for $i=1$ there is \cl{a subsequence $(n_k)_{k \in \NN}$ of $\NN$ such that $\lambda_{n_k,1} \to \lambda_1$}. 
\cl{One can iteratively find subsequences of the resulting subsequences such that for all $i = 1, \dots, m$} the convergences of $\cl{\lambda_{s_k,i} \to \lambda_i}$ and $\cl{x_{s_k,i} \to x_i}$ hold true \cl{for a subsequence $(s_k)_{k \in \NN}$ of $\NN$}. Because of $\psi \in W^{1,2}(\RR)$ function $\psi$ is (Lebesque) almost everywhere continuous. Hence, \cl{$f_{s_k} \to \hat{f}$} holds almost everywhere for $\hat{f} \doteq \sum_{i=1}^n \cl{\lambda_i \tau_{x_i}} \psi$. It is known, see \cite[Theorem 3.12]{rudin1987real}, that this is sufficient for $f = \hat{f}$ almost everywhere. \cl{A similar argument can be used to show that $\mathcal{M}$ is sequentially compact. Then, since $\mathcal{M}$ is a metric space, it is also compact; not just locally, but globally.}
		
		We also need to prove that for all $f \in \mathcal{M}$ 
		\begin{equation}
			\liminf_{n \to \infty}\frac{\QX( B(f,R \, h_n))}{h_n^{2m}}  >0 \qquad\QX\text{-a.s.}
		\end{equation}
		We show that for $\tilde{f} \doteq \sum_{i=1}^m \tilde{\lambda}_i \tau_{\tilde{x}_i}\psi$ from
			\begin{equation}\label{eq:vector-norm}
			\begin{aligned}
			\Bigg\|
				\begin{bmatrix}
					\lambda\\
					x
				\end{bmatrix}
				- 
				\begin{bmatrix}
					\tilde{\lambda}\\
					\tilde{x}
				\end{bmatrix}
			\Bigg\|_2 < h
			\end{aligned}
		\end{equation}
		it follows that $\|f- \tilde{f}\|_2 < C\, h$. Recall that on finite dimensional vector spaces all norms are equivalent, thus one can use $\norm{\cdot}_\infty$ instead of $\norm{\cdot}_2$, i.e., there exists $c > 0$ such that
		\begin{equation}
			\begin{aligned}
				\Bigg\|
				\begin{bmatrix}
					\lambda\\
					x
				\end{bmatrix}
				-
				\begin{bmatrix}
					\tilde{\lambda}\\
					\tilde{x}
				\end{bmatrix}
				\Bigg\|_\infty \leq\;
				c \,
				\Bigg\|
				\begin{bmatrix}
					\lambda\\
					x
				\end{bmatrix}
				- 
				\begin{bmatrix}
					\tilde{\lambda}\\
					\tilde{x}
				\end{bmatrix}
				\Bigg\|_2
			\end{aligned}
		\end{equation}
		holds. In conclusion, if \eqref{eq:vector-norm} holds, then
		\begin{equation}
			\begin{aligned}
				\Bigg\|
				\begin{bmatrix}
					\lambda\\
					x
				\end{bmatrix}
				-
				\begin{bmatrix}
					\tilde{\lambda}\\
					\tilde{x}
				\end{bmatrix}
				\Bigg\|_\infty \leq c \, h.
			\end{aligned}
		\end{equation}
		We use the notion of moduli of continuity to prove that $\norm{f -\tilde{f}}_2 < C \, h$. Recall that
		\begin{equation}
			\omega_\psi^{(2)}(h) \doteq \sup_{|x| \leq h} \int |\psi(t+ x) - \psi(t)|^2 \dd t.
		\end{equation}
		It is known \cite[Section 5.8,Theorem 3]{evans2010partial} that for \cl{$\psi \in  W^{1,2}(\RR)$ there is a} constant $M_\psi$ with
        \vspace{-3mm}
		\begin{equation}\label{ineq:sobolev}
			\omega_\psi^{(2)}(h) \leq M_\psi h^2.
		\end{equation}
		Using the triangle inequality and \eqref{ineq:sobolev} one can gain
		\begin{equation}
			\begin{aligned}
				\norm{f-\tilde{f}}_2 &= \norm{ \sum_{i=1}^m\lambda_i \tau_{x_i} \psi -\sum_{i=1}^m\tilde{\lambda}_i \tau_{\tilde{x}_i} \psi }_2\\
				&\leq \norm{ \sum_{i=1}^m\lambda_i \tau_{x_i} \psi- \sum_{i=1}^m\tilde{\lambda}_i \tau_{x_i} \psi + \sum_{i=1}^m\tilde{\lambda}_i \tau_{x_i} \psi -\sum_{i=1}^m\tilde{\lambda}_i \tau_{\tilde{x}_i} \psi }_2\\
				&\leq \norm{ \sum_{i=1}^m\lambda_i \tau_{x_i} \psi- \sum_{i=1}^m\tilde{\lambda}_i \tau_{x_i} \psi}_2 + \norm{\sum_{i=1}^m\tilde{\lambda}_i \tau_{x_i} \psi -\sum_{i=1}^m\tilde{\lambda}_i \tau_{\tilde{x}_i} \psi }_2 \\    
				&\leq \sum_{i=1}^m |\lambda_i - \tilde{\lambda}_i| \, \norm{\tau_{x_i} \psi }_2 + \sum_{i=1}^m |\tilde{\lambda}_i|\, \norm{\tau_{x_i} \psi - \tau_{\tilde{x}_i} \psi}_2 \\
				&\leq h \, c \, m \, \norm{\psi}_2 + M \sum_{i=1}^m \bigg(\int_\RR |\psi(t- x_i) - \psi(t - \tilde{x}_i)|^2 \dd t \bigg) ^{1/2}\\
				&\leq h \, c \, m \, \norm{\psi}_2 + M \sum_{i=1}^m \bigg(\int_\RR |\psi(s) - \psi(s + (x_i - \tilde{x}_i) )|^2 \dd t \bigg) ^{1/2}\\
				& \leq h\, c \, m \, \norm{\psi}_2 + M \sum_{i=1}^m \sqrt{\omega_\psi^{(2)}(c\, h)}\\
				& \leq  h \, c \, m\,\Big(\hspace{0.3mm}\norm{\psi}_2 + M \, \sqrt{M_\psi}\, \Big),
			\end{aligned}
		\end{equation}
		hence our claim is satisfied with $C= m\, c \big(\norm{\psi}_2 + M \sqrt{ M_\psi}\,\big)$.
		
		We showed that $B(f, C \, h_n))) \supseteq \phi(B([\lambda\tr, x\tr]\tr, h_n))$, therefore
        \vspace{-1mm}
		\begin{equation}
		\begin{aligned}
			&\QX(B(f, R \, h_n)))  \geq \PX\circ\phi^{-1} \big[ \phi(B([x\tr, \lambda\tr]\tr, R/C \, h_n))\big]\\
			&\geq \PX (B( [\lambda\tr, x\tr]\tr, R/C \, h_n)).
		\end{aligned}
		\end{equation}
		By \cite[Lemma 2.2]{devroye1981almost} we conclude that
\begin{equation}
			\begin{aligned}
				&\liminf_{n\to \infty} \frac{\PX (B( [\lambda\tr, x\tr]\tr,R/C \, h_n))}{h_n^{2m}} = \frac{1}{g([\lambda\tr, x\tr]\tr)} > 0 \quad \quad \PX-\text{a.s.}
			\end{aligned}
		\end{equation}
		and 
		\begin{equation}
			\begin{aligned}
				&\liminf_{n\to \infty}\frac{\QX(B(f,R\, h_n))}{h_n^{2m}} > 0 \quad \quad \QX-\text{a.s.}
			\end{aligned}
		\end{equation}
		By Theorem \ref{thm:result} the corollary follows.
	\end{proof}

        \section{Discussion}

        In this paper, we have introduced a new recursive estimator for the conditional kernel mean map, which is a key object to represent conditional distributions in Hilbert spaces.\ We have proved the weak and strong $L_2$ consistency of the introduced scheme in the Bochner sense, under general structural conditions.\ We have considered a locally compact Polish input space and assumed only the knowledge of a kernel on the (measurable) output space. Our recursive algorithm uses a smoother sequence on the metric input space and three learning sequences, each of these having a particular role in the iterative settings. Learning rate $(a_n)$ adjusts the weights of the current estimate term and the update term in each step. Sequences $(h_n)$ and $(r_n)$ are used to adaptively shrink the smoother functions w.r.t.\ the intrinsic measuretic structure of the metric space. Our main assumption concerns the probabilities of small balls, i.e., we assume that the measure of small balls with radius $h_n$ are comparable to $r_n$ when $n$ goes to infinity. We have generalized a result of \citet[Theorem 25.1]{gyorfi2002distribution} for locally compact Polish input spaces and Hilbert space valued regressions. We have considered three important application domains of our recursive estimation scheme. Our consistency results are universal in case of Euclidean input spaces and complete Riemannian manifolds, because the small ball probability assumption can be satisfied by $(h_n)$ and $(r_n)$ for all probability distributions.\ A Hilbert space valued input set (containing functions) was also considered to illustrate the generality of the framework.

	\section*{Acknowledgements}
        The authors warmly thank L{\'a}szló Gerencsér for his many comments and suggestions regarding the paper. His advises helped us to shape the final manuscript in many ways.
 
	The research was supported by the European Union project RRF-2.3.1-21-2022-00004 within the framework of the Artificial Intelligence National Laboratory. This work has  also been supported by the TKP2021-NKTA-01 grant of the National Research, Development and Innovation Office (NRDIO), Hungary. The authors  acknowledge the professional support of the Doctoral Student Scholarship Program of the Cooperative Doctoral Program of the Ministry of Innovation and Technology, as well, financed from an NRDIO Fund.

	\appendix
	
	\section*{Appendix A.}\label{app:a}
	
	\subsection*{The Bochner Integral}
	
	Bochner integral is an extended integral notion for Banach space valued functions. Let $(\CG, \norm{\cdot}_\CG)$ be a Banach space with the Borel $\sigma$-field of $\CG$ and $( \Omega, \mathcal{A}, \BP)$ be a measure space with a finite measure $\BP$. Similarly to the Lebesgue integral, the Bochner integral is introduced first for simple $\Omega \to \CG$ type functions of the form $f(\omega) = \sum_{j=1}^n\BI(\omega \in A_j )\, g_j$, where $A_j \in \mathcal{A}$ and $g_j \in \CG$ for all $j \in [n]$, by 
    \begin{equation}
        \int_\Omega f \dd \BP \,\doteq \, \sum_{j=1}^n \,\BP(A_j)\, g_j.
    \end{equation}
    Then, the integral is extended to the abstract completion of these simple functions with respect to the norm defined by the following Lebesgue integral
	\begin{equation}
\norm{f}_{L_1(\BP;\CG)} \,\doteq\, \int_{\Omega} \norm{f}_\CG \dd \BP.
	\end{equation}
	
	\begin{definition}
		A function $f:\Omega \to \CG$ is called strongly measurable $($Bochner-measurable\hspace{0.2mm}$)$ if there exists a sequence $(f_n)_{n \in \NN}$ of simple functions converging to $f$ almost everywhere.
	\end{definition}
	By the Pettis theorem \citep[Theorem 1.1.20]{hytonen2016analysis}, this is  equivalent to requiring the weak measurability of $f$ and that $f(\Omega) \subseteq \CG$ is separable.
	
	\begin{definition}
		Let $1 \leq p < \infty$. Then, the vector valued $\mathcal{L}_p(\Omega, \mathcal{A}, \BP; \CG)$-space is defined as
		\begin{equation}
			\mathcal{L}_p(\Omega, \mathcal{A}, \BP; \CG )\! =\mathcal{L}(\BP; \CG )\! \doteq\! \Big\{ f\! : \Omega \to \CG \,\big|\, f \,\textit{is strongly measurable and} \int \!\norm{f}_\CG^p\! \dd \BP < \infty \Big\}.\!
		\end{equation}
		A semi-norm on $\mathcal{L}_p(\Omega, \mathcal{A}, \BP; \CG)$ is defined by
		\begin{equation}
			\norm{f}_{L_p(\BP; \CG)}\, \doteq \,\Big( \int \norm{f}_\CG^p \dd \BP\Big)^\frac{1}{p}\!\!\!, \quad \textit{for }\;f \in \mathcal{L}_p(\Omega, \mathcal{A}, \BP; \CG).
		\end{equation}
	\end{definition}
	Taking equivalence classes on $\mathcal{L}_p(\Omega, \mathcal{A}, \BP; \CG)$ w.r.t. semi-norm $\norm{\cdot}_{L_p}$ yields the vector valued $L_p$ space (Bochner space) denoted by $L_p(\Omega, \mathcal{A}, \BP; \CG)$, or simply by $L_p(\Omega, \BP; \CG)$. 
	\begin{proposition}
		Space  $(L_p(\Omega, \mathcal{A}, \BP; \CG), \norm{\cdot}_{L_p(\BP;\CG)})$ is a Banach space for $1  \leq p < \infty$.
	\end{proposition}
	Several of the well-known theorems about Lebesgue integrals can be extended to Bochner integrals, see \citep{dinculeanu2000vector, hytonen2016analysis, pisier2016martingales}.
	\begin{proposition}[Bochner's inequality] \label{ineq:bochner} For all $f \in L_1 (\Omega, \mathcal{A}, \BP; \CG)$, we have
		\begin{equation}
			\norm{\,\int f \dd \BP\, }_\CG \leq \int \norm{f}_\CG \dd \BP.
		\end{equation}
	\end{proposition}
	Moreover, $f$ is Bochner integrable if and only if the right hand side is finite, however, recall that $f$ is required to be strongly measurable to be included in $L_1 (\Omega, \mathcal{A}, \BP; \CG)$. Furthermore, observe that $\int \norm{f}_\CG \dd \BP = \norm{f}_{L_1(\BP;\CG)}$, hence the Bochner integral is a continuous operator from $L_1(\Omega, \mathcal{A}, \BP; \CG)$ to Banach space $\CG$. By definition, the Bochner integral is also linear.
	\begin{proposition} \label{prop:bounded_lin} Let $E$ and $F$ be Banach spaces and $T: E \to F$ be a bounded linear operator. Then, for all $f \in L_1 (\BP; E)$, we have that $T \circ f \in L_1(\BP;F)$ and
		\begin{equation}
			T\Big(\int f \dd \BP \Big) = \int T(f) \dd \BP.
		\end{equation}
		In particular, for a bounded linear functional $x^*: \CG \to \RR$, we have
		\begin{equation}
			\Big\langle \int_\Omega f(\omega) \dd \BP(\omega) , x^*\Big\rangle = \int_\Omega  \langle f(\omega) ,x^*\rangle \dd \BP(\omega),
		\end{equation}
		where $\langle g,x^*\rangle \doteq  x^*(g)$.
	\end{proposition}
	
	Now, let $(\CG, \langle \cdot, \cdot\rangle_\CG)$ be a Hilbert space. 
	By Proposition \ref{prop:bounded_lin} for all $g \in \CG$,  we have 
	\begin{equation}\label{eq:7}
		\Big\langle \int_\Omega f(\omega) \dd \BP(\omega) , g\Big\rangle_\CG = \int_\Omega  \langle \mu(\omega) ,g\rangle_\CG \dd \BP(\omega).
	\end{equation}
Moreover, it can be proved that $L_2(\Omega, \mathcal{A}, \BP; \CG)$ is also a Hilbert space with the following inner product
\begin{equation}
		\langle \mu_1 , \mu_2 \rangle_{L_2(\Omega, \BP;\CG)} \,\doteq \int \langle \mu_1(\omega) , \mu_2(\omega) \rangle_\CG \dd \BP(\omega).
	\end{equation}

	\subsection*{Conditional Expectation in Banach Spaces}
	
	Let $\BP$ be a probability measure and $\CB$ a sub-$\sigma$-algebra of $\CA$. The conditional expectation with respect to $\CB$ is defined for any $f \in L_1(\Omega, \CA, \BP; \CG)$ as the $\BP$-a.s.\ unique element $\EE\big[\hspace{0.2mm} f \,|\, \CB\hspace{0.3mm}\big]  \doteq \nu \in L_1(\Omega,\CB, \BP;\CG)$ satisfying
	\begin{equation}
		\int_B \nu  \dd \BP\, = \int_B f \dd \BP, \quad \text{for all } B \in \CB.  
	\end{equation}
	It can be shown that the conditional expectation is a (well-defined) bounded operator from $L_1(\Omega,\CA,\BP;\CG)$ to $L_1(\Omega, \CB, \BP;\CG)$ \cite[Theorem 2.6.23]{hytonen2016analysis}, \cite[Proposition 1.4]{pisier2016martingales}. The nonexpanding property also holds in the $L_p(\BP;\CG)$ sense for all $1 \leq p < \infty$, see \cite[Corollary 2.6.30]{hytonen2016analysis}, that is for all $f \in  L_p(\BP;\CG)$ we have
	\begin{equation}\label{eq:nonexpansive}
		\norm{\,\EE\big[\,f \,|\, \CB\,\big]\, }_\CG^p \,\leq\, \EE\big[\,\norm{f}_\CG^p\,|\,\CB \,\big],
	\end{equation}
	and by Bochner's inequality $\norm{\EE\big[\hspace{0.2mm}f \,|\, \CB\hspace{0.3mm}\big]}_{L_p(\Omega; \CG)} \leq \norm{f}_{L_p(\Omega;\CG)}$.
	Furthermore, by \cite[Prop. 2.6.31.]{hytonen2016analysis} we can take out the $\CB$-measurable term from the conditional expectation. In particular, if $\CG$ is a Hilbert space, then for all $g \in \CG$ and $f \in L_2(\Omega, \CA, \BP; \CG)$ we (a.s.), have that
	\begin{equation}\label{cont-innerprod}
		\langle\, \EE\big[\, \mu \,|\, \CB\,\big], g \,\rangle_\CG \,= \,\EE \big[\, \langle \mu, g \rangle_\CG\,|\,\CB\,\big].
	\end{equation}
	
	\section*{Appendix B.}

	In the proof of the generalized Stone's theorem a continuous approximation of the conditional kernel mean map is needed. The existence of such approximation is ensured by the following theorem.
	\begin{theorem}\label{denseness}
		Let $\BX$ be  a locally compact Hausdorff space equipped with a Radon measure\,\footnote{\cl{A Radon measure is a regular Borel measure.}} $\QX$ and let $\CG$ be  a separable Hilbert space. The space of compactly supported continuous functions, $\mathcal{C}_c(\BX; \CG)$, is dense in $L_2(\BX, \QX; \CG)$, that is for all $\mu \in L_2(\BX, \QX; \CG)$ and $\varepsilon > 0$, there exists a $\tilde{\mu} \in \mathcal{C}_c(\BX; \CG)$ such that 
		\begin{equation}
			\int_\BX \norm{\mu(x) - \tilde{\mu}(x) }_\CG^2 \dd \QX (x)\, <\, \varepsilon.
		\end{equation}
	\end{theorem}
        A similar result is proved for $\RR^d$ in \cite[Lemma 1.2.31]{hytonen2016analysis}. They mention in a remark that their lemma can be generalized for locally compact spaces, but the complete proof of their claim were not found by us in the literature. Nevertheless, the presented argument has a similar flavor as the ones for scalar valued $L_p$ spaces.
        \smallskip
	\begin{proof}
	First, we prove that we can approximate $\mu$ with \cl{simple functions} in $L_2(\BX, \QX; \CG)$. Let $\mu \in L_2(\BX , \QX; \CG)$ and $\varepsilon >0$. There exists a sequence of  \cl{simple functions} $(\mu_n)_{n\in \NN}$ such that $\mu_n \to \mu$ a.s., i.e., $\mu_n(x) \doteq \sum_{i=1}^{k_n} g_i \,\BI(x \in A_i)$ where $g_i \in \CG$ and $A_i$ is measurable with $\QX(A_i) < \infty$ for all $i \in [k_n]$. Let $B_n \doteq \{x \in \BX: \norm{\mu_n(x)}_\CG \leq 2 \norm{\mu(x)}_\CG\}$ and 
	\begin{equation}
		h_n(x) \doteq  \BI(x \in B_n) \, \mu_n(x).
	\end{equation}
	One can see that the $(B_n)_{n\in \NN}$ sequence consists of measurable sets and $h_n \to \mu$ almost everywhere. We show that $h_n \to \mu$ in $L_2(\BX, \QX;\CG)$. Because of 
	\begin{equation}
		\norm{h_n}^2_{L_2(\BX;\CG)} =   \int_{B_n} \norm{ \mu_n(x) }_\CG^2 \dd \QX(x) \leq 2^2 \, \int_{\cl{\BX}} \norm{\mu(x)}^2_\CG \dd \QX (x),
	\end{equation}
	we have $h_n \in L_2(\BX, \QX;\CG)$ for $n \in \NN$. In addition, almost everywhere we have
	\begin{equation}
		\norm{\mu(x) - h_n(x) }_\CG^2 \leq (3\hspace{0.3mm} \norm{\mu(x)}_\CG)^2,
	\end{equation}
	therefore, by the dominated convergence theorem
	\begin{equation}
	\begin{aligned}
		\lim_{n\to \infty} \norm{\mu - h_n }^2_{L_2(\BX, \QX;\CG)} &=\lim_{n \to \infty} \int_\BX \norm{\mu(x) - h_n(x) }_\CG^2\dd \QX(x)\\
		& = \int_\BX \lim_{n \to \infty}  \norm{\mu(x) - h_n(x) }_\CG^2\dd \QX(x) = 0.
	\end{aligned}
	\end{equation}
	In conclusion, for any $\varepsilon >0$ there exists a simple function $h \in L_2(\BX, \QX; \CG)$ such that $\norm{\mu - h}_{L_2(\BX, \QX;\CG)} < \varepsilon$. Therefore, it is sufficient to approximate $h$ with a continuous function on a compact support. Let $h \in L_2(\BX, \QX; \CG)$ be 
	\begin{equation}
		\begin{aligned}
			h(x) = \sum_{i=1}^N g_i  \, \BI(x \in A_i),
		\end{aligned}
	\end{equation}
	where $g_i \in \CG$ and $\QX(A_i) < \infty$ for $i \in [N]$. The main idea is to approximate the indicator functions, $\{g_i \cdot \BI(x \in A_i)\}_{i=1}^n$, separately. Let us fix $i=1$ and $\varepsilon > 0$.
	Since $\QX$ is Radon and $A_1$ is measurable there exists a compact set $K\subseteq \BX$ such that $K \subseteq A_1$ and
    \begin{equation}
        \QX(A_1 \setminus K) < \frac{\varepsilon^2}{2 N^2\, \norm{g_1}_\CG^2},
    \end{equation}
    and there exists an open set $U$ such that $A_1 \subseteq U$ with
    \begin{equation}
        \QX(U\setminus A_1) <\frac{\varepsilon^2}{2 N^2\, \norm{g_1}_\CG^2}.
    \end{equation}
    The application of Lemma \ref{lemma:LCH} yields that there exists an open set $E$ such that $\bar{E}$ is compact and $K \subseteq E \subseteq \bar{E} \subseteq U$. Because of Lemma \ref{lemma:Ury} there is a continuous function $f_1$ such that $f_1\arrowvert_K = 1$ and $f_1\arrowvert_{\BX\setminus E} =0$, from which it follows that $f_1$ has a compact support. For $h_1(x)\doteq g_1 \cdot \BI(x \in A_1)$ and $\tilde{h}_1(x)\doteq g_1 \cdot f_1(x)$ we have
	\begin{equation}
	\begin{aligned}
		&\int_\BX \norm{h_1(x) - \tilde{h}_1(x)}_\CG^2 \dd \QX (x) = \int_\BX \norm{g_1 \,\BI(x \in A_1) - g_1 \, f_1(x)}_\CG^2 \dd \QX(x)\\
		&=  \norm{g_1}_\CG^2  \, \int_\BX \big( \BI(x \in A_1) -  f_1(x) \big)^2 \dd \QX(x) \leq \norm{g_1}_\CG^2  \, \int_{E\setminus K} 1  \dd \QX(x) \\
		&= \norm{g_1}_\CG^2 \, \QX(E \setminus K ) \leq \norm{g_1}_\CG^2 \, \QX(U \setminus K ) < \frac{\varepsilon^2}{N^2}.
	\end{aligned}	
	\end{equation}
	Similarly, one can construct $\tilde{h}_i$ which is close to $h_i$ in $L_2(\BX, \QX; \CG)$ for all $i \in [N]$. Let $\tilde{\mu} \doteq \sum_{i=1}^N \tilde{h}_i$. Obviously, $\tilde{\mu}$ is continuous and has a compact support. Furthermore, by the triangle inequality we have
	\begin{equation}
	\begin{aligned}
		&\sqrt{\int_\BX \norm{h(x) - \tilde{\mu}(x)}_\CG^2\dd \QX (x)} = \sqrt{\int_\BX \Big\|\sum_{i=1}^N g_i \big(\BI(x \in A_i) - f_i(x) \big) \Big\|_\CG^2\dd \QX (x)}\\
		& \leq\sum_{i=1}^N \sqrt{\int_\BX \big\| g_i \big(\BI(x \in A_i) - f_i(x) \big) \big\|_\CG^2\dd \QX(x)} \leq \sum_{i=1}^N \sqrt{\frac{\varepsilon^2}{N^2}}  \,\cl{=} \,\varepsilon,
	\end{aligned}
	\end{equation}
	which finishes the proof of the theorem.
	\end{proof}

	\begin{lemma}\label{lemma:sajat}
Let $\{(\BX_i, \mathcal{X}_i)\}_{i=1}^3$ be measurable spaces, 
let $X_i$, for $i\in \{1,2,3\}$, be $\BX_i$-valued independent random elements on a probability space $(\Omega, \CA, \BP)$ with push-forward measure $Q_i$.
        Let $f:\BX_1 \times \BX_2 \times \BX_3 \to \RR$ and $g: \BX_1 \times \BX_2 \to \RR$ be measurable functions such that $f(X_1,X_2, X_3)$ and $g(X_1, X_2)$ are in $L_1(\Omega, \BP)$. If for $Q_{\scriptscriptstyle 1}$-almost all $x \in \BX_1$ it holds that
		\begin{equation}\label{eq:inequality}
			\EE\big[\, f(x_1,X_2, X_3) \,|\, X_2\,\big] \,\leq\, g(x_1,X_2),
		\end{equation} 
		then almost surely
		\begin{equation}
			\EE\bigg[ \int_{\BX_1} f(x_1,X_2, X_3) \dd Q_{\scriptscriptstyle 1}(x_1) \,\Big|\, X_2\,\bigg] \,\leq\, \int_{\BX_1} g(x_1,X_2) \dd Q_{\scriptscriptstyle 1} (x_1).
		\end{equation} 
	\end{lemma}
	\begin{proof}
		Integrating out both sides in \eqref{eq:inequality} w.r.t.\ $Q_{\scriptscriptstyle 1}$ yields
		\begin{equation}
		\begin{aligned}
		  &\int_{\BX_1} \EE\Big[ f(x_1,X_2, X_3) \, | \, X_2 \Big] \dd  Q_{\scriptscriptstyle 1} (x_1) \leq \int_{\BX_1} g(x_1, X_2) \dd Q_{\scriptscriptstyle 1}(x_1)= \EE[g(X_1, X_2)].
		\end{aligned}
        \end{equation} 
        In addition, the left hand side is
        \begin{equation}
        \begin{aligned}
            &\EE\Big[ \int_\BX f(x_1,X_2, X_3) \dd  Q_{\scriptscriptstyle 1} (x_1) \,\big|\, X_2\Big] = \int_{\BX_3} \int_{\BX_1} f(x_1,X_2, x_3) \dd  Q_{\scriptscriptstyle 1} (x_1) \dd  Q_{\scriptscriptstyle 3} (x_3) \\
&= \int_{\BX_1} \int_{\BX_3} f(x_1,X_2, x_3) \dd  Q_{\scriptscriptstyle 3} (x_3) \dd  Q_{\scriptscriptstyle 1} (x_1) = \int_{\BX_1} \EE\Big[ f(x_1,X_2, X_3) \, | \, X_2 \Big] \dd  Q_{\scriptscriptstyle 1} (x_1),
        \end{aligned}
        \end{equation}
        \cl{because} of Fubini's theorem.
	\end{proof}
\section*{Appendix C}
	\noindent The details of the proof for strong consistency in Theorem \ref{thm:sajat} are presented in this section.
        \smallskip
        \begin{proof}
                  By the weak consistency of $(\mu_n)_{n \in \NN}$, Fubini's theorem and \eqref{eq:nonexpansive} we have
     	\begin{equation}
     	\begin{aligned}
     		\int \norm{\, \EE [\mu_n(x)] - \mu_*(x)\,}_\CG^2 \dd \QX (x)\, \leq\, \EE \int \norm{\, \mu_n(x) - \mu_*(x)\,}_\CG^2 \dd \QX (x) \to 0.
     	\end{aligned}
     	\end{equation}
     	Because of $(a+b)^2 \leq 2 a^2 + 2b^2$, we have
     	\begin{equation}\label{eq:inprob}
     		\begin{aligned}
     			 &\EE \int \norm{ \mu_n(x) - \EE [\mu_n(x)]}_\CG^2 \dd \QX (x) \\
     			 &\leq 2 \Big( \EE \int \norm{ \mu_n(x) - \mu_*(x)}^2_\CG \dd \QX(x) + \int \norm{\mu_*(x) - \EE[\mu_n(x)]}_\CG^2 \dd \QX(x) \Big)\to 0.
     		\end{aligned}
     	\end{equation}
     	Since
     	\begin{equation}
     		\begin{aligned}
     			&\int \norm{ \mu_n(x) - \mu_*(x)}_\CG^2 \dd \QX (x) \\
     			&\leq 2 \Big( \int \norm{ \mu_n(x) - \EE[\mu_n(x)]}^2_\CG \dd \QX(x) + \int \norm{\mu_*(x) - \EE[\mu_n(x)]}_\CG^2 \dd \QX(x) \Big)
     		\end{aligned}
     	\end{equation}
     	holds, for the strong consistency of $(\mu_n)$ it is sufficient to prove that
      \begin{equation}\label{eq:convergence}
          \int \norm{ \mu_n(x) - \EE[\mu_n(x)]}^2_\CG \dd \QX(x) \to 0
      \end{equation}
      almost surely. By \eqref{eq:inprob} the integral admits a finite limit which must agree with the weak limit (which is zero).
Let us consider the following expansion:
     	\begin{equation} \label{eq:normsquare}
     	\begin{aligned}
     		&\mu_{n+1}(x) - \EE[\mu_{n+1}(x)] = \mu_n(x) - \EE [ \mu_n(x)] \\[1mm]
     		&\quad - a_{n+1} \big( \mu_n(x) k_{n+1}(x,X_{n+1}) - \EE \big[ \mu_n(x) k_{n+1}(x,X_{n+1}) \big] \big)\\[1mm]
     		&\quad + a_{n+1} \big( \ell(\cdot,Y_{n+1}) k_{n+1}(x,X_{n+1}) - \EE \big[ \ell(\cdot,Y_{n+1}) k_{n+1}(x,X_{n+1}) \big] \big).
     	\end{aligned}
     	\end{equation}
     	Let $\CF_n \doteq \sigma( X_1, Y_1, \dots, X_n, Y_n)$ for $n \in \NN$. Taking the conditional expectation of the norm square of \eqref{eq:normsquare} w.r.t. $\CF_n$ yields
     	\begin{align}
     			&\EE \Big[ \norm{\mu_{n+1}(x) - \EE[\mu_{n+1}(x)]}_\CG^2 \,|\, \CF_n \Big] = \norm{\mu_n(x) - \EE [ \mu_n(x)]}_\CG^2\nonumber\\
     			&+ a_{n+1}^2 \EE \Big[ \norm{  \mu_n(x) k_{n+1}(x,X_{n+1}) - \EE \big[ \mu_n(x) k_{n+1}(x,X_{n+1}) \big]}_\CG^2\, |\, \CF_n\Big]\nonumber\\
     			&+ a_{n+1}^2 \EE \Big[ \norm{ \ell(\cdot,Y_{n+1}) k_{n+1}(x,X_{n+1}) - \EE \big[  \ell(\cdot,Y_{n+1}) k_{n+1}(x,X_{n+1}) \big]}_\CG^2\, |\, \CF_n\Big]\nonumber\\
     			&-2 a_{n+1}\EE\Big[ \big\langle \mu_n(x) -\EE[\mu_n(x)], \mu_n(x) k_{n+1} (x,X_{n+1}) - \EE[\mu_n(x) k_{n+1}(x, X_{n+1})] \big\rangle_\CG \,\big|\,\CF_n\Big]\nonumber\\
     			&+2 a_{n+1} \EE\Big[ \big\langle \mu_n(x) -\EE[\mu_n(x)], \ell(\cdot, Y_{n+1}) k_{n+1} (x,X_{n+1}) - \EE[\ell(\cdot, Y_{n+1}) k_{n+1}(x, X_{n+1})] \big\rangle_\CG \,\big|\,\CF_n\Big]\nonumber\\
     			& - 2 a_{n+1}^2 \EE\Big[ \big\langle \mu_n(x) k_{n+1} (x,X_{n+1}) - \EE[\mu_n(x) k_{n+1}(x, X_{n+1})],\nonumber\\
     			& \hspace{3cm} \ell(\cdot, Y_{n+1}) k_{n+1} (x,X_{n+1}) - \EE[\ell(\cdot, Y_{n+1}) k_{n+1}(x, X_{n+1})] \big\rangle_\CG \,\big|\,\CF_n\Big]\nonumber\\
     			& = I_1 + I_2 + I_3 + I_4 + I_5 + I_6,
     	\end{align}
     	where $I_i$ are defined respectively for $i \in \{1,\dots,6\}$.
     	We bound these terms separately using the measurability condition, the independence of the sample and condition $(ii)$. Our main tool to prove almost sure convergence uses almost supermartingales, Theorem \ref{thm:robbins-ziegmund} from \citep{robbins1971convergence}. Therefore, our main task in this section is to bound each $I_i$ for $i \in \{1,\dots,6\}$ by $(1 + \alpha_n) \big( \norm{\mu_n(x) - \EE[ \mu_n(x)]}_\CG^2\big) + \beta_n$ where $\alpha_n$ and $\beta_n$ are small enough, i.e.,
        \vspace{-1mm}
        \begin{equation}
            \sum_{n \in \NN} \EE [\alpha_n] < \infty\qquad \text{ and }\qquad \sum_{n \in \NN} \EE [\beta_n] < \infty
        \end{equation}
        holds. Clearly, $I_1$ is bounded by itself. For the second term we have
        \begin{equation} 
     		\begin{aligned}
     			I_2 &\doteq a_{n+1}^2 \EE \Big[ \norm{  \mu_n(x) k_{n+1}(x,X_{n+1}) - \EE \big[ \mu_n(x) k_{n+1}(x,X_{n+1}) \big]}_\CG^2\, |\, \CF_n\Big]\\
     			&= a_{n+1}^2 \bigg( \EE \Big[ \norm{ \mu_n(x) k_{n+1}(x, X_{n+1}) }_\CG^2 \,|\,\CF_n\Big] + \norm{\EE [ \mu_n(x)] \EE [k_{n+1}(x,X_{n+1})]}_\CG^2\\
     			& \quad -2  \langle \mu_n(x) \EE[ k_{n+1}(x,X_{n+1})], \EE [ \mu_n(x)] \EE[ k_{n+1} (x,X_{n+1})]\rangle_\CG \bigg)\\     
     			& = a_{n+1}^2 \bigg( \norm{ \mu_n(x)}_\CG^2 \big( \EE [ k_{n+1}^2(x, X_{n+1})] - \big(\EE[k_{n+1}(x,X_{n+1})]\big)^2\big)\\
     			& \quad + \norm{\mu_n(x)}_\CG^2 \big(\EE[k_{n+1}(x,X_{n+1})]\big)^2 - 2 \langle \mu_n(x), \EE [ \mu_n(x)] \rangle_\CG \big(\EE[ k_{n+1} (x,X_{n+1})]\big)^2 \\
     			& \quad + \norm{\EE[\mu_n(x)]}_\CG^2  \big(\EE[ k_{n+1} (x,X_{n+1})]\big)^2\bigg)\\
     			& = a_{n+1}^2 \bigg( \norm{ \mu_n(x)}_\CG^2 \big( \EE [ k_{n+1}^2(x, X_{n+1})] - \big(\EE[k_{n+1}(x,X_{n+1})]\big)^2\big)\\
     			& \quad +\norm{\mu_n(x) - \EE[\mu_n(x)]}_\CG^2 \big(\EE[k_{n+1}(x,X_{n+1})]\big)^2\bigg)\\
     			& \leq \frac{H^2(0) a_{n+1}^2}{r_{n+1}^2} \big( \norm{\mu_n(x)}_\CG^2 + \norm{\mu_n(x) - \EE[ \mu_n(x)]}_\CG^2\big)
     		\end{aligned}
     	\end{equation}
        because recall that $r_{n+1} k(x, X_{n+1}) \leq H(0)$. We also used that $(X_{n+1}, Y_{n+1})$ is independent of $\CF_n$ and $\mu_n(x)$ is measurable w.r.t. $\CF_n$. The third term can be bounded as follows
     	\begin{equation} 
     		\begin{aligned}
     			I_3 &\doteq a_{n+1}^2 \EE \Big[ \norm{   \ell(\cdot,Y_{n+1}) k_{n+1}(x,X_{n+1}) - \EE \big[  \ell(\cdot,Y_{n+1}) k_{n+1}(x,X_{n+1}) \big]}_\CG^2\, |\, \CF_n\Big]\\
     			& = a_{n+1}^2 \bigg( \EE \Big[\norm{\ell(\cdot, Y_{n+1}) k_ {n+1}(x,X_{n+1})}_\CG^2\Big] + \norm{\EE \big[  \ell(\cdot,Y_{n+1}) k_{n+1}(x,X_{n+1}) \big]}_\CG^2\\
     			& \quad - 2 \EE \Big[ \langle  \ell(\cdot,Y_{n+1}) k_{n+1}(x,X_{n+1}), \EE \big[  \ell(\cdot,Y_{n+1}) k_{n+1}(x,X_{n+1}) \big]\rangle_\CG \Big]\bigg)\\
     			& = a_{n+1}^2 \Big( \EE \Big[\ell(Y_{n+1}, Y_{n+1}) k_{n+1}^2(x,X_{n+1})\Big] - \norm{\EE \big[  \ell(\cdot,Y_{n+1}) k_{n+1}(x,X_{n+1}) \big]}_\CG^2\Big)\\
    			& \leq \frac{H^2(0) a_{n+1}^2}{r_{n+1}^2} \EE [ \ell(Y,Y)].
     		\end{aligned}
     	\end{equation}
     	The fourth term is less than $0$ because
     	\begin{align} 
     			I_4 &= -2 a_{n+1}\EE\Big[ \big\langle \mu_n(x) -\EE[\mu_n(x)], \mu_n(x) k_{n+1} (x,X_{n+1}) - \EE[\mu_n(x) k_{n+1}(x, X_{n+1})] \big\rangle_\CG \,\big|\,\CF_n\Big]\nonumber\\
     			& = -2 a_{n+1} \big\langle \mu_n(x) -\EE[\mu_n(x)], \mu_n(x) \EE [k_{n+1} (x,X_{n+1})]- \EE[\mu_n(x)] \EE[k_{n+1}(x, X_{n+1})] \big\rangle_\CG\nonumber\\
     			& = -2 a_{n+1} \norm{\mu_n(x) - \EE[\mu_n(x)]}_\CG^2 \EE [ k_{n+1}(x,X_{n+1})] \leq 0.
     	\end{align}
     	It is easy to see that $I_5= 0$, because of independence we have
        \begin{equation}		
        \begin{aligned}
     	\EE\Big[ \big\langle \mu_n(x) -\EE[\mu_n(x)], \ell(\cdot, Y_{n+1}) k_{n+1} (x,X_{n+1}) - \EE[\ell(\cdot, Y_{n+1}) k_{n+1}(x, X_{n+1})] \big\rangle_\CG \,\big|\,\CF_n\Big]\\
     			 = \big\langle \mu_n(x) -\EE[\mu_n(x)], \EE[\ell(\cdot, Y_{n+1}) k_{n+1} (x,X_{n+1})] - \EE[\ell(\cdot, Y_{n+1}) k_{n+1}(x, X_{n+1})] \big\rangle_\CG.
     		\end{aligned}
                \end{equation}
     	For the last term we use the Cauchy-Schwartz inequality and $2ab \leq a^2 + b^2$ to show that
     	\begin{equation} 
     	\begin{aligned}
     			I_6 &= - 2 a_{n+1}^2 \EE\Big[ \big\langle \mu_n(x) k_{n+1} (x,X_{n+1}) - \EE[\mu_n(x) k_{n+1}(x, X_{n+1})],\\[-1.5mm]
     			& \hspace{3cm} \ell(\cdot, Y_{n+1}) k_{n+1} (x,X_{n+1}) - \EE[\ell(\cdot, Y_{n+1}) k_{n+1}(x, X_{n+1})] \big\rangle_\CG \,\big|\,\CF_n\Big]\\[-2mm]         
     			&= - 2 a_{n+1}^2 \bigg( \EE\Big[ \big\langle \mu_n(x) k_{n+1} (x,X_{n+1}) ,\ell(\cdot, Y_{n+1}) k_{n+1} (x,X_{n+1}) \big\rangle_\CG  \,\big|\,\CF_n\Big]\\
     			& \quad - \langle \EE[\mu_n(x) k_{n+1}(x, X_{n+1}) ], \EE[\ell(\cdot, Y_{n+1}) k_{n+1}(x, X_{n+1})\,|\, \CF_n] \big\rangle_\CG\\
     			& \quad -  \langle \EE[\mu_n(x) k_{n+1}(x, X_{n+1})\,|\, \CF_n ],\EE[\ell(\cdot, Y_{n+1}) k_{n+1}(x, X_{n+1})] \big\rangle_\CG\\
     			& \quad + \langle \EE[\mu_n(x) k_{n+1}(x, X_{n+1}) ], \EE[\ell(\cdot, Y_{n+1}) k_{n+1}(x, X_{n+1})\,|\, \CF_n] \big\rangle_\CG\bigg)\\
     			&  = -2 a_{n+1}^2 \EE\big[ k_{n+1}^2 (x,X_{n+1}) \big\langle \mu_n(x)  ,\ell(\cdot, Y_{n+1}) \big\rangle_\CG \,\big|\,\CF_n\big]\\
     			& \quad + 2 a_{n+1}^2 \EE[ k_{n+1}(x, X_{n+1}) ] \langle \mu_n(x) ,\EE[\ell(\cdot, Y_{n+1}) k_{n+1}(x, X_{n+1})] \big\rangle_\CG\\
     			&\leq 2 a_{n+1}^2\EE [ k_{n+1}^2 (x,X_{n+1}) \norm{\mu_n(x)}_\CG \sqrt{\ell(Y_{n+1}, Y_{n+1})}\,|\,\CF_n ]\\
     			& \quad + 2 a_{n+1}^2 \EE[ k_{n+1}(x, X_{n+1}) ] \norm{\mu_n(x)}_\CG \norm{\EE[\ell(\cdot , Y_{n+1}) k_{n+1}(x, X_{n+1})]}_\CG\\
     			&\leq  2 a_{n+1}^2 \norm{\mu_n(x)}_\CG \EE \bigg[ \frac{H^2(0)}{r_{n+1}^2} \sqrt{\ell(Y_{n+1}, Y_{n+1})}\,|\,\CF_n \bigg]\\
     			& \quad +  2 a_{n+1}^2 \norm{\mu_n(x)}_\CG \frac{H^2(0)}{r_{n+1}^2}  \EE[\norm{\ell(\cdot , Y_{n+1})}_\CG]\\			
     			& = 4 \frac{H^2(0) a_{n+1}^2}{r_{n+1}^2} \norm{\mu_n(x)}_\CG \EE\Big[ \sqrt{\ell(Y,Y)}\Big] \leq 2\frac{H^2(0) a_{n+1}^2}{r_{n+1}^2} \big( \norm{\mu_n(x)}_\CG^2 + \EE[ \ell(Y,Y)] \big).
     	\end{aligned}
        \end{equation}
    	By summarizing these upper bounds we have almost surely that
        \vspace{-1mm} 
    		\begin{align}
    			& \EE \Big[ \norm{\mu_{n+1}(x) - \EE[\mu_{n+1}(x)]}_\CG^2 \,|\, \CF_n \Big] \\
    			&\leq \bigg( 1 +\frac{H^2(0) a_{n+1}^2}{r_{n+1}^2} \bigg) \norm{\mu_n(x) - \EE[\mu_n(x)]}_\CG^2 + \frac{3 H^2(0) a_{n+1}^2}{r_{n+1}^2} \big( \norm{\mu_n(x)}_\CG^2 + \EE[\ell(Y,Y)]\big)\nonumber
    		\end{align}
    	for all $x \in \BX$. By Lemma \ref{lemma:sajat} one can integrate w.r.t. $\QX$ to prove that
        \vspace{-1mm}
    	\begin{equation} 
    		\begin{aligned}   
    			& \EE \Big[ \int \norm{\mu_{n+1}(x) - \EE[\mu_{n+1}(x)]}_\CG^2 \dd \QX(x)\,|\, \CF_n \Big] \\
    			&\leq \bigg( 1 +\frac{H^2(0) a_{n+1}^2}{r_{n+1}^2} \bigg) \int \norm{\mu_n(x) - \EE[\mu_n(x)]}_\CG^2 \dd \QX(x) \\
    			&\quad + \frac{3 H^2(0) a_{n+1}^2}{r_{n+1}^2} \Big( \int \norm{\mu_n(x)}_\CG^2 \dd \QX(x) + \EE[\ell(Y,Y)]\Big).
    		\end{aligned}
    	\end{equation}
    	Observe that $\EE[\ell(Y,Y)] < \infty$ and $ \EE\big[ \int \norm{\mu_n(x)}_\CG^2 \dd \QX(x)\big]$ is convergent, hence also bounded. Consequently, from 
    	\begin{equation*} 
    		\begin{aligned}     		
    			&\sum_{n=1}^\infty \frac{H^2(0) a_{n+1}^2}{r_{n+1}^2} < \infty
    		\end{aligned}
    	\end{equation*}
    	it follows that $\int \norm{\mu_{n+1}(x) - \EE[\mu_{n+1}(x)]}_\CG^2 \dd \QX(x)$ converges almost surely by Theorem \ref{thm:robbins-ziegmund}\hspace*{0.25mm}. The almost surely constant \cl{limit value} is zero because of our previous argument.
        \end{proof}
	
    \section*{Appendix D}
	In this appendix, for convenience, we state the main theorems that are crucial for our proofs. An important result is the \cl{generalized} SLLN for Hilbert space valued random elements. The following theorem is from the book of \citet[Theorem 3.2.4]{taylor1978stochastic}.
	\noindent
	\begin{theorem} If $(X_n)_{n \in \NN}$ is  a sequence of independent random elements in a separable Hilbert space such that 
    \vspace{-2mm}
		\begin{equation}
			\sum_{n=1}^\infty \frac{\Var{(X_n)}}{n^2} < \infty,
		\end{equation} 
		where $\Var{(X_n)} \doteq \EE \Big[\norm{ X_n - \EE X_n}^2\Big]$, then 
		\begin{equation}
			\norm{\frac{1}{n}\sum_{i=1}^n (X_i - \EE X_i ) } \xrightarrow{\,a.s.\,} 0.
	\end{equation}
	\end{theorem}

        We referred to \cite[Theorem 1.9]{ljung2012stochastic} to prove the consistency of our recursive (unconditional) kernel mean embedding estimate of a marginal distribution.
	\begin{theorem}\label{thm:RM}
		Let $\CG$ be a real separable Hilbert space endowed with the Borel $\sigma$-algebra and $\Lambda: \CG \to \CG$ be measurable and  $\mu \in \CG$ $($usually but not necessarily $\Lambda(\mu)= 0\hspace{0.2mm})$. Let further $\mu_n$, $W_n$, $H_n$, $V_n$ be $\CG$-valued random elements for $n \in \NN$ with $\mu_0 =0$,
		\begin{equation}
			\mu_{n+1}\, \doteq\, \mu_n - a_n(\Lambda(\mu_n) - W_n) \qquad\text{and}\qquad W_n= H_n + V_n,
		\end{equation}
		where $a_n \geq 0$, $\sum_n a_n^2 < \infty$ and $\sum_n a_n = \infty$. Assume that
		\begin{enumerate}[label=(\roman*)]
			\item There exists $c >0$ such that for all $g \in \CG$ we have $\norm{\Lambda(g)}_\CG \leq c(1 + \norm{g}_\CG)$. 
			\item For all $K \in [1,\infty)$ we have $\inf\{ \langle \Lambda (g) , g- \mu\rangle_\CG\,|\, g \in \CG \;\text{with} \;\frac{1}{K} \leq \norm{g- \mu}_\CG \leq K \} >0$.
			\item For all $n \in \NN$ random elements $H_n$ and $V_n$ are square integrable with 
			\begin{equation}
				\begin{aligned}
					&\sum_n a_n\,\EE\norm{H_n}_\CG < \infty, \qquad \sum_n a_n^2\, \EE [\norm{H_n}_\CG^2] < \infty,\\[1mm]
					&\EE\big[\,V_n \,|\,\mu_1, H_1, V_1, \dots , H_{n-1}, V_{n-1}\, \big]= 0, \quad\text{and} \quad \sum_n a_n^2\, \EE[\norm{V_n}_\CG^2 ]< \infty.
				\end{aligned}
            \vspace*{-4mm}
			\end{equation}
		\end{enumerate}
		Then, $\mu_n \to \mu$ $(n \to \infty)$ almost surely.
	\end{theorem}

    \noindent We applied the following forms of the \cl{fundamental} topological results from \citep[Lemma 5.1]{nagy2001real} and \citep[Theorem 5.1]{nagy2001real} to prove Theorem \ref{denseness}.
	\noindent
	\begin{lemma}\label{lemma:LCH}
	Let $(\BX, \mathcal{T})$ be a locally compact Hausdorff space. Let $K$ be a compact subset of $\BX$, and let $U$ be an open set, with $K \subseteq U$. Then, there exists an open set $E$ such that $\bar{E}$ is compact and $K \subseteq E \subseteq \widebar{E} \subseteq U$.
	\end{lemma}

	\begin{lemma} [Urysohn] \label{lemma:Ury} Let $(\BX, \mathcal{T})$ be a locally compact Hausdorff space. Let $K$ and $F$ be disjoint subsets of $\BX$, where $K$ is compact and $F$ is closed. Then, there exists a continuous function $f: \BX \to [0,1]$ such that $f\arrowvert_K = 1$ and $f\arrowvert_F =0$.
	\end{lemma}

        \noindent We also used the following well-known facts from real analysis, see \citep[Corollary 2.4.6]{dudley2002real}, \citep[Theorem 4.15]{rudin1976principles} and \citep[Theorem 2]{ruder1966silverman}
	\begin{theorem}[Heine-Cantor]\label{thm:heine}
	A continuous function from a compact metric space to any
	metric space is uniformly continuous.
	\end{theorem}
	\begin{theorem}[Weierstrass]
 	If $f$ is a continuous mapping of a compact metric space $\BX$ into
 	$\RR$, then $f$ is bounded.
	\end{theorem}

	\begin{lemma}[Toeplitz]\label{lemma:Toeplitz}
		Let $(a_{n,i})_{n,i \in \NN}$ be a doubly indexed array of real numbers such that
		\begin{equation}
		\begin{aligned}
			\sup_{n \in \NN} \,\sum_{i =1}^\infty | a_{n,i}| < \infty, \qquad\text{and}\qquad\lim_{n \to \infty} a_{n,i} = 0,
		\end{aligned}
		\end{equation}
		for any $i \in \NN$. If $x_n \to x$ as $n \to \infty$, then\; $\lim_{n \to \infty} \sum_{i=1}^\infty a_{n,i} x_i = 0$.
		
	\end{lemma}
	We applied a simplified version of the theorem of \citet[Theorem 1]{robbins1971convergence} to prove the strong consistency of our recursive algorithm.
    \begin{theorem}[almost supermartingale convergence]\label{thm:robbins-ziegmund}
		Let $\CF_1 \subseteq \CF_2 \subseteq \dots \subseteq \CF$ be a filtration, i.e., a nondecreasing sequence of $\sigma$-algebras. For every $n \in \NN$ let $V_n$, $\alpha_n$ and $\beta_n$ be non-negative $\CF_n$ measurable random variables such that
		\begin{equation}
			\EE [\, V_{n+1} \,|\, \CF_n\,]  \,\leq\, (1 + \alpha_n) V_n + \beta_n
		\end{equation}
		holds almost surely. If $\sum_{n=1}^\infty \EE[ \alpha_n ]< \infty$ and\, $\sum_{n=1}^\infty \EE [\beta_n] < \infty$ hold almost surely, then $(V_n)_{n\in \NN}$ converges almost surely to a finite limit.
	\end{theorem}
	
	\noindent The proof of Corollary \ref{cor:riemannian} used the following theorem \citep[Theorem 6.13]{lee2006riemannian}.
	\begin{theorem}[Hopf-Rinow]
		A connected Riemannian manifold is
		geodesically complete if and only if it is complete as a metric space.
	\end{theorem}

        \noindent We used the Lipschitzness of smooth maps with bounded derivatives between Riemannian manifolds w.r.t.\ the geodesical distance. A published version of the following proof was not found by us, however, the argument is straightforward.
\begin{lemma}\label{meanvalue-R}
	    Let $(M,g)$ and $(N,h)$ be Riemannian manifolds and $f$ a smooth $M \to N$ map. If there exists $C > 0$ such that $\norm{Df(x)} \leq C$, for all $x \in M$\cl{, where $\norm{\cdot}$ denotes the operator norm,} then for all $x, y \in M$ one has
	    \begin{equation}
	        d_N(f(x),f(y)) \leq C \, d_M(x,y),
	    \end{equation}
	    where $d_M$ and $d_N$ are the geodesical distances.
	\end{lemma}

	\begin{proof}
	    Let $x, y \in M$. For all $\varepsilon > 0$ there exists a smooth curve $\gamma:[0,1] \to M$ such that $\gamma(0) = x$ and $\gamma(1) = y$ and 
	    \begin{equation}
	        \int_0^1 \norm{\gamma'(t)}_g \leq d_M(x,y) + \varepsilon.
	    \end{equation}
	    Observe that $f \circ \gamma$ is a smooth curve from $f(x)$ to $f(y)$ in $N$, hence
	    \begin{equation}
	    \begin{aligned}
	        d_N(f(x),f(y)) &\leq \int_0^1 \norm{(f\circ \gamma)' (t)}_h \dd t
	        = \int_0^1 \norm{Df(\gamma(t)) (\gamma'(t))}_h \dd t\\
	        & \leq \int_0^1 \norm{Df(\gamma(t))} \, \norm{\gamma'(t)}_g \dd t \leq C\,( d_M(x,y) + \varepsilon)
	    \end{aligned}
	    \end{equation}
	    holds. It is true for all $\varepsilon > 0$ thus $d_N(f(x),f(y)) \leq C\cdot d_M(x,y)$ also holds.
	\end{proof}

	\vskip 0.2in
	\bibliography{JMLR_bib}
	
\end{document}